\documentclass{article}

\usepackage{microtype}
\usepackage{graphicx}
\usepackage{subcaption}
\usepackage{booktabs} 
\usepackage{multirow}
\usepackage{dirtytalk}
\usepackage{enumitem}

\usepackage{hyperref}

\usepackage[accepted]{icml2026}

\usepackage{amsmath}
\usepackage{amssymb}
\usepackage{mathtools}
\usepackage{amsthm}

\DeclareFontEncoding{LS1}{}{}
\DeclareFontSubstitution{LS1}{stix}{m}{n}
\DeclareSymbolFont{stixletters}{LS1}{stix}{m}{it}
\DeclareMathAccent{\backvec}{\mathord}{stixletters}{"91}

\usepackage[capitalize,noabbrev]{cleveref}

\theoremstyle{plain}
\newtheorem{theorem}{Theorem}[section]
\newtheorem{proposition}[theorem]{Proposition}

\theoremstyle{definition}
\newtheorem{definition}[theorem]{Definition}

\theoremstyle{remark}

\usepackage[textsize=tiny]{todonotes}

\icmltitlerunning{Stable and Near-Reversible Diffusion ODE Solvers for Image Editing}

\begin{document}

\twocolumn[
  \icmltitle{Stable and Near-Reversible Diffusion ODE Solvers for Image Editing}



  \icmlsetsymbol{equal}{*}

  \begin{icmlauthorlist}
    \icmlauthor{Barbora Barancikova}{com}
    \icmlauthor{Daniil Shmelev}{math}
    \icmlauthor{Cristopher Salvi}{math}
  \end{icmlauthorlist}

  \icmlaffiliation{com}{Department of Computing, Imperial College London, London, United Kingdom}
  \icmlaffiliation{math}{Department of Mathematics, Imperial College London, London, United Kingdom}

  \icmlcorrespondingauthor{Barbora Barancikova}{b.barancikova23@imperial.ac.uk}

  \icmlkeywords{Machine Learning, ICML}

  \vskip 0.3in
]



\printAffiliationsAndNotice{}  

\begin{abstract}
The inversion of diffusion models plays a central role in image editing. Algebraically reversible ODE solvers provide an appealing approach to diffusion inversion for text-guided image editing, by eliminating the inversion error inherent in DDIM-based editing pipelines. However, empirical results indicate that reversibility alone is insufficient. As edits require larger semantic or visual changes, reversible diffusion solvers often exhibit instabilities and suffer sharp drops in output quality. In this paper, we show that the trade-off between exact reversibility and numerical stability manifests empirically as a trade-off between background preservation and prompt alignment in image editing. We then investigate the use of near-reversible Runge--Kutta methods as a more stable alternative to exactly reversible diffusion schemes. When combined with a vector-field smoothing strategy, the resulting approach improves edit fidelity, remains stable under large edits, and largely retains the background-preservation benefits of reversible solvers.
\end{abstract}

\section{Introduction}
Diffusion probabilistic models \citep{sohl2015deep, ho2020denoising, song2019generative, song2020score}
have become a standard tool for text-guided image editing \citep{Rombach_2022_CVPR, nichol2021glide}.
A common approach is \textit{inversion-based editing}---given a real image, the probability-flow ODE \citep{song2020score} is integrated backwards in time to a latent noise state conditional on a source prompt, after which the forward ODE is integrated under a new target prompt to obtain an edited image \citep{song2020denoising, mokady2023null, ju2023direct}.

\begin{figure}[H]
    \centering
    \begin{subfigure}[t]{\linewidth}
        \centering
        \includegraphics[width=\linewidth]{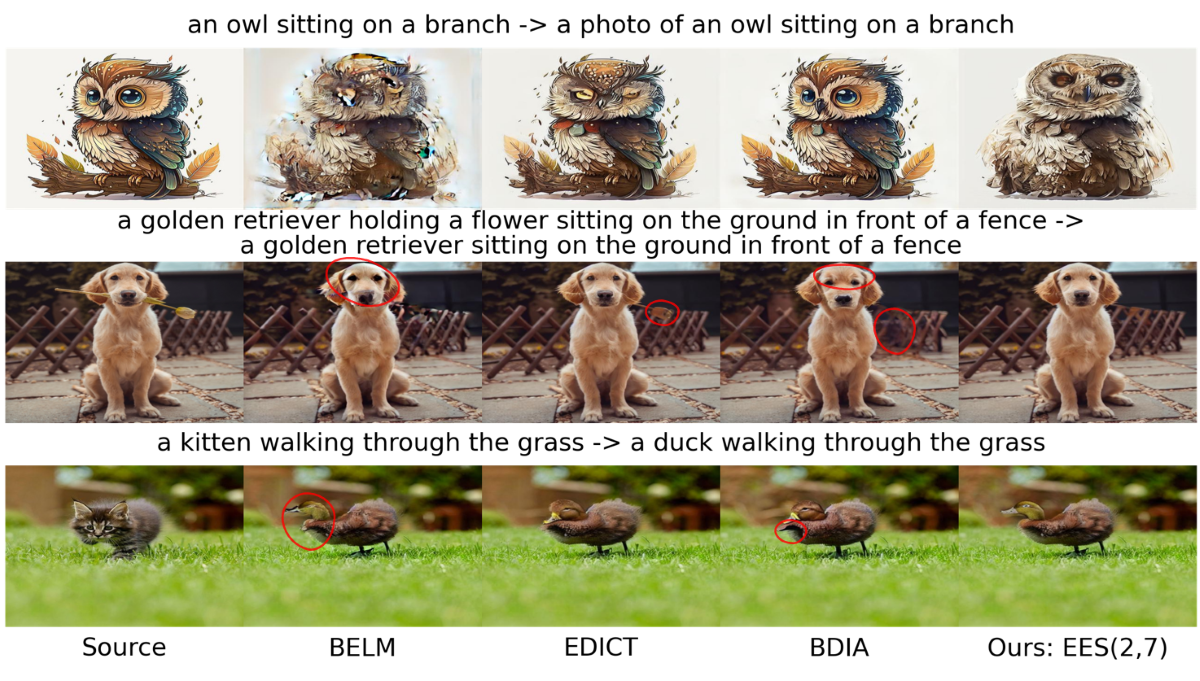}
        \caption{\textbf{Small Edits}}
        \label{fig:example_figure_small}
    \end{subfigure}
    \begin{subfigure}[t]{\linewidth}
        \centering
        \includegraphics[width=\linewidth]{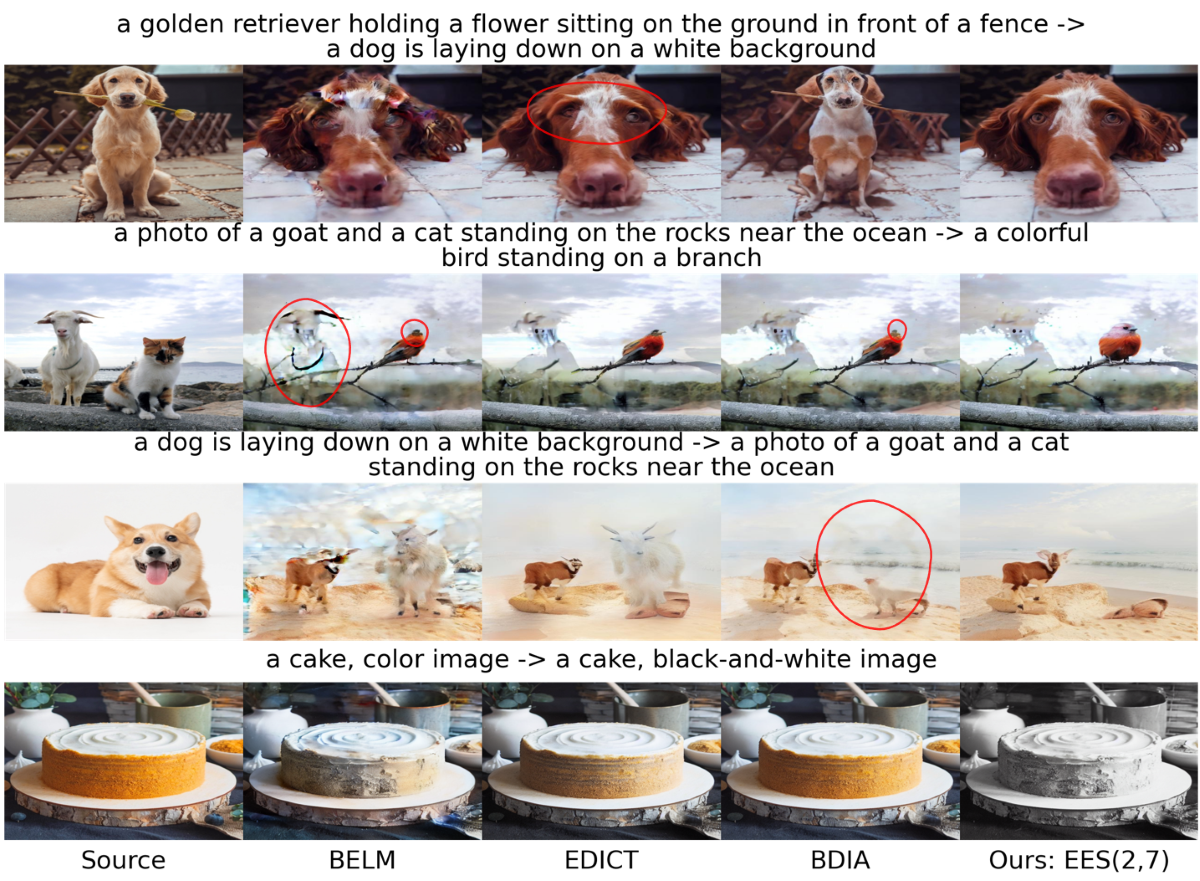}
        \caption{\textbf{Large Edits}}
        \label{fig:example_figure_large}
    \end{subfigure}

    \caption{
    Qualitative comparison of (near) reversible solvers on image editing tasks.
    (\textbf{\subref{fig:example_figure_small}}) Standard PIE-Bench editing tasks (\cref{sec:experiments_editing}). Here, we use Smooth Diffusion for BELM, EDICT, and EES.
    (\textbf{\subref{fig:example_figure_large}}) PIE-Bench images with large prompt deviations (\cref{sec:large_edits_and_stability}) and greyscale conversion (\cref{sec:greyscale}). Here, we use Smooth Diffusion for all methods. See extended grids in \cref{fig:edit_grid_owl_extended,fig:edit_grid_retriever_extended,fig:large_edits_retriever_extended,fig:black_and_white_extended}.
    }
    \label{fig:example_figure}
\end{figure}

In this setting, the choice of numerical solver is vital, since errors resulting from inexact inversion compound through the editing process, especially under classifier-free guidance \citep{ho2021classifierfree}. For instance, classical DDIM inversion incurs a systematic trajectory mismatch, since the reverse update is not the algebraic inverse of the forward update. A long line of work aims to reduce the resulting diffusion inversion error \citep{han2023improving, miyake2025negative, guo2024smooth}. Recently, \textit{algebraically reversible} (or simply \textit{reversible}) ODE solvers such as EDICT \citep{wallace2023edict}, BDIA \citep{zhang2024exact}, BELM \citep{wang2024belm}, and Rex \citep{blasingame2025rex} have been proposed as an alternative to DDIM inversion, yielding improved reconstructions and background preservation in editing pipelines. However, our empirical evidence suggests that reversibility alone is not enough. On standard editing tasks, reversible solvers exhibit a clear trade-off between background preservation and prompt alignment. This becomes particularly apparent when the editing trajectory departs substantially from the inversion trajectory, in which case many reversible solvers suffer sharp drops in edit quality.

To benchmark solver performance for image editing, we distinguish two evaluation settings. The first is the widely used \textit{small edits} setting of the PIE-Bench image editing benchmark introduced by \citet{ju2023direct}, where edits are usually local to a subset of the prompt and image (see \cref{sec:experiments_editing}). In addition, we introduce a second \textit{large edits} regime, where the edited prompt induces a large semantic or visual change, forcing the sampling
trajectory away from the original inversion path (see \cref{sec:large_edits_and_stability}). Figure~\ref{fig:example_figure} previews the key findings.

In \cref{sec:stability}, we argue that the failure modes of existing reversible schemes, illustrated in Figure~\ref{fig:example_figure}, can be attributed to numerical stability. To address these issues, we propose using the recently introduced class of \textit{near-reversible Explicit and Effectively Symmetric (EES) Runge--Kutta methods} \citep{shmelev2025explicit} as diffusion solvers. EES schemes relax exact reversibility, instead enforcing only \emph{approximate} symmetry up to an acceptably low error tolerance. This relaxation allows them to retain an explicit Runge--Kutta form and benefit from much more favourable stability properties. Since the inversion error depends on the regularity of the underlying trajectory, we additionally combine EES schemes with vector-field smoothing strategies previously proposed for DDIM inversion. Together, this yields a family of solvers that is competitive with exactly reversible solvers on standard editing benchmarks and substantially more stable under challenging edits.

In addition to proposing EES solvers, our paper clarifies the broader design
space of reversible diffusion solvers for editing. We investigate guidance and
smoothing heuristics such as Smooth Diffusion \citep{guo2024smooth}, Negative
Prompt Inversion \citep{miyake2025negative}, and Proximal Guidance
\citep{han2023improving}, which have previously been treated mainly as
orthogonal DDIM-based techniques. We show that these methods can also yield
substantial gains for reversible solvers, and characterise when they help.
We also benchmark Reversible Heun and McCallum--Foster methods, originally
introduced outside the diffusion models literature, and run extensive ablations
over ODE parametrisations, step schedules, and solver formulations.

\textbf{We summarise our main contributions as follows:}
\vspace{-0.4em}
\begin{enumerate}[leftmargin=*, itemsep=0.2em, topsep=0.2em, parsep=0pt, partopsep=0pt]
    \item We provide a side-by-side evaluation of (near) reversible diffusion solvers for image editing, revealing a trade-off curve between background preservation and prompt alignment.

    \item We propose using near-reversible EES Runge--Kutta methods for inversion-based diffusion editing. For \textbf{small edits}, when combined with vector-field smoothing, these methods retain much of the background-preservation benefit of exactly reversible solvers while improving edit fidelity.

    \item For \textbf{large edits}, we show that EES is the most robust (near) reversible family in our study, maintaining output quality under large trajectory deviations where most exactly reversible solvers degrade substantially. 
\end{enumerate}
\vspace{-0.4em}

\section{Preliminaries}
\subsection{Reversible Solvers}
\label{sec:reversible_solvers}

In the setting of an autonomous ODE $dy_t = f(y_t)dt$, a one-step method
\begin{equation}\label{eq:ode}
    y_{n+1} = y_n + h\Phi_h(y_n)
\end{equation}
is said to be \emph{symmetric} if $\Phi_{-h} = \Phi_h^{-1}$. That is, a step of the method $\Phi$ applied to $y_{n+1}$ with a negative step size exactly recovers $y_n$. Classical symmetric numerical integration methods, such as those of Runge--Kutta type, have been studied extensively for their favourable long-term behaviour, especially in the context of Hamiltonian systems \cite{hairer2006geometric, feng2010symplectic, Chartier2015}. More recently, symmetric methods have proven useful in the training of neural differential equations \cite{chen2018neural}, where the ability to recover the full solution trajectory from the terminal point allows for memory-efficient backpropagation.\par

A major drawback of classical symmetric methods lies in their low efficiency. It is well-known that any symmetric Runge--Kutta scheme is necessarily implicit, and more generally any symmetric parasitism-free general linear method cannot be explicit \cite{butcher2016symmetric}. This inefficiency makes classical schemes unsuitable for applications in machine learning. To solve this problem, several new \emph{algebraically reversible} schemes, typically of low-order, were introduced. The asynchronous leapfrog integrator (ALF) \cite{zhuang2021mali} for Neural ODEs overcomes the problem of implicit schemes by tracking an auxiliary state as part of the integration. A similar approach is taken by the Reversible Heun method \cite{kidger2021efficient}, where the technique of tracking an auxiliary state is applied to form a reversible version of Heun's method for SDEs.\par

Whilst both methods prove to be efficient, they are inherently unstable. This instability is a significant bottleneck in certain practical applications \cite{shmelev2025explicitneural, zhang2021path}. 
In \cite{mccallum2024efficient}, a new method was proposed for transforming any existing one-step integration method into one which is reversible. Whilst these methods offer increased stability over the ALF and Reversible Heun solvers, their stability domains are still typically much smaller than those of classical methods. This method was later adapted to the case of diffusion models, where the approach forms the basis for Rex solvers \cite{blasingame2025rex}.

A potential solution to the issues of stability is offered by Explicit and Effectively Symmetric (EES) Runge--Kutta schemes \citep{shmelev2025explicit}. Whilst symmetric Runge--Kutta methods are necessarily implicit, the authors show that it is possible to derive explicit schemes which are \emph{almost} symmetric, up to a given tolerance level. The resulting schemes are virtually indistinguishable from perfectly symmetric schemes for sufficiently smooth problems and maintain stability domains which are comparable to those of classical schemes such as RK4 and RK5. It was shown in \citet{shmelev2025explicitneural} that EES schemes produce superior results to Reversible Heun and McCallum-Foster methods when applied to stiff or highly volatile Neural SDEs.

\subsection{Diffusion ODE and Inversion-Driven Image Editing}
\paragraph{Diffusion models.}
Diffusion probabilistic models (DPMs) \citep{sohl2015deep, ho2020denoising, song2020score} define a forward noising process
that gradually perturbs data \(x_0 \sim q_0\) into a tractable Gaussian at time \(T\).
In the variance-preserving (VP) setting, the forward marginals satisfy
\begin{equation}
q_{t|0}(x_t \mid x_0)=\mathcal{N}(x_t;\,\alpha_t x_0,\,\sigma_t^2 I),
\qquad \alpha_t,\sigma_t>0,
\label{eq:vp_marginal}
\end{equation}
or equivalently \(x_t=\alpha_t x_0+\sigma_t \varepsilon\) with \(\varepsilon\sim\mathcal{N}(0,I)\).
We fix \(\alpha_t^2 + \sigma_t^2 = 1\). DPMs aim to learn a denoiser \(\epsilon_\theta(x_t,t,c)\) that predicts the noise \(\varepsilon\) from \(x_t\), optionally dependent on some condition \(c\).
At test time, generation can be performed by solving a reverse-time stochastic differential equation (SDE) \citep{song2020score} that replaces the intractable score term \(\nabla_x \log q_t(x)\)
with the model prediction.

\paragraph{Probability-flow ODE.}
A convenient deterministic alternative to the reverse-time SDE is the probability-flow ODE \citep{song2020score},
which has the same marginals \(q_t\) as the reverse-time SDE but removes the stochastic term.
For the VP parametrisation \eqref{eq:vp_marginal}, using the standard score approximation
\(\nabla_x \log q_t(x_t) \approx -\hat\epsilon_\theta(x_t,t,c)/\sigma_t\),
the probability-flow ODE can be written in the form
\begin{equation}
\frac{d x_t}{d t}
= \frac{d\log\alpha_t}{dt}\left(x_t - \frac{\hat\epsilon_\theta(x_t,t,c)}{\sigma_t}\right),
\quad t\in[0,T].
\label{eq:pf_ode_logalpha}
\end{equation}
Sampling corresponds to drawing \(x_T \sim \mathcal{N}(0,\sigma_T^2 I)\) and integrating \eqref{eq:pf_ode_logalpha} from \(t=T\) to \(t=0\).

\paragraph{Classifier-free guidance.}
When sampling a diffusion model conditional on \(c\) (often a text prompt), classifier-free guidance \citep{ho2021classifierfree} has proven particularly effective for improving conditional sample fidelity. A guided noise prediction is formed by combining conditional and unconditional outputs
(\(c_{\texttt{null}}:=\texttt{<empty string>}\)):
\begin{equation}
\hat\epsilon_\theta(x_t,t,c)
= g\epsilon_\theta(x_t,t,c)+(1-g)\epsilon_\theta(x_t,t,c_{\texttt{null}}),
\label{eq:cfg}
\end{equation}
where \(g\ge 0\) is the \emph{guidance weight}.

\paragraph{Inversion-driven image editing.}
A common approach to text-guided image editing with diffusion models is to first invert a real image into the model’s latent noise space and then resample under a new, ``edited''  prompt \citep{song2020denoising, hertz2022prompt, mokady2023null}. Given a real input image \(x_0\) with source prompt \(c_{\texttt{src}}\) and a target prompt \(c_{\texttt{trg}}\), we perform editing via ODE inversion by integrating \eqref{eq:pf_ode_logalpha} in the backward direction from \(t=0\) to \(t=T\) with \(c_{\texttt{src}}\) to obtain \(x_T\), and then integrating \eqref{eq:pf_ode_logalpha} forward from \(t=T\) to \(t=0\) starting from the same \(x_T\) but using \(c_{\texttt{trg}}\), yielding the edited image \(\tilde x_0\).

\subsection{Inversion Samplers}
To edit images as described above, we need to be able to invert a diffusion solver, i.e., run it in the backward direction. In this section, we introduce several inversion strategies, ranging from DDIM inversion to reversible solvers.

\paragraph{DDIM inversion.}
Define \(\bar x(t):=x(t)/\alpha_t\) and \(\bar\sigma(t):=\sigma_t/\alpha_t\).
As summarised in \citet{wang2024belm}, the probability-flow ODE can be written in the form:
\begin{equation}
\frac{d\bar x}{d\bar\sigma}(t)
=
\bar\epsilon_\theta(\bar x(t),\bar\sigma(t),c),
\quad
\bar\epsilon_\theta(\bar x,\bar\sigma,c)
:=
\epsilon_\theta(\alpha_t \bar x, t, c).
\label{eq:diffusion_ivp_barsigma}
\end{equation}

Let \(t_N=T>\cdots>t_0=0\) and \(\bar\sigma_N>\cdots>\bar\sigma_0\), and define steps
\(h_i:=\bar\sigma_i-\bar\sigma_{i-1}>0\).
Then the DDIM sampler corresponds to an explicit Euler step:
\begin{equation}
\bar x_{i-1}
=
\bar x_i - h_i\,\bar\epsilon_\theta(\bar x_i,\bar\sigma_i,c).
\label{eq:ddim_euler}
\end{equation}
In the original \(x\)-variable this is equivalently
\begin{equation}
x_{i-1}
=
\frac{\alpha_{i-1}}{\alpha_i}x_i
+
\Bigl(\sigma_{i-1}-\frac{\alpha_{i-1}}{\alpha_i}\sigma_i\Bigr)\,
\epsilon_\theta(x_i,t_i,c).
\label{eq:ddim}
\end{equation}
\emph{DDIM inversion} involves running \eqref{eq:ddim} in the reverse direction from \(t_0\) to \(t_N\) using the same scheme \eqref{eq:ddim_euler}.

\paragraph{DDIM is not algebraically reversible.}
It is clear that, since the Euler update rule does not satisfy the symmetry property in \cref{sec:reversible_solvers}, the DDIM inversion approach is not exact---the forward and backward maps evaluate the network $\epsilon_\theta$ at different points \(x_i\) and \(x_{i-1}\). In inversion-driven editing, this mismatch accumulates along the trajectory, leading to low background preservation and a decrease in overall edit quality. Multiple works also observe that the degradation becomes substantially worse under classifier-free guidance, which amplifies per-step errors \citep{wallace2023edict, hertz2022prompt}. 

\paragraph{Mitigating DDIM inversion error.}
Many works reduce DDIM inversion drift via inference-time or model-level interventions. Smooth Diffusion \citep{guo2024smooth} fine-tunes the model to encourage smoother latent geometry, improving inversion and downstream edits. Other methods adjust conditioning during inversion, e.g. null-text inversion (NTI) \citep{mokady2023null}, negative-prompt inversion (NPI) \citep{miyake2025negative} and proximal guidance \citep{han2023improving}. We summarise the latter two in \cref{appendix:smoothing_methods}. Some methods also store and reuse information from the source trajectory to guide editing, e.g.\ Prompt-to-Prompt \citep{hertz2022prompt}, MasaCtrl \citep{cao2023masactrl}, pix2pic-zero \citep{parmar2023zero}, plug-and-play \citep{tumanyan2023plug}, and Direct Inversion \citep{ju2023direct}. These techniques do not make DDIM algebraically reversible, but often reduce the practical impact of inversion errors.

\paragraph{Reversible diffusion samplers.}
Motivated by the inexact inversion of DDIM, several works design samplers whose forward and backward updates are algebraic inverses.

\textbf{EDICT} \citep{wallace2023edict} introduces an auxiliary diffusion state \(y_i\) coupled with \(x_i\) via a mixing coefficient \(p\in(0,1)\). Writing
\(a_i := \alpha_{i-1}/\alpha_i\) and \(b_i := \sigma_{i-1} - \frac{\alpha_{i-1}}{\alpha_i}\sigma_i\),
one EDICT step takes the form
\begin{equation}
\begin{aligned}
x_i^{\mathrm{inter}} &= a_i x_i + b_i\,\epsilon_\theta(y_i, t_i, c), \\
y_i^{\mathrm{inter}} &= a_i y_i + b_i\,\epsilon_\theta(x_i^{\mathrm{inter}}, t_i, c), \\
x_{i-1} &= p\,x_i^{\mathrm{inter}} + (1-p)\,y_i^{\mathrm{inter}}, \\
y_{i-1} &= p\,y_i^{\mathrm{inter}} + (1-p)\,x_{i-1}.
\end{aligned}
\label{eq:edict_core}
\end{equation}

\textbf{BDIA} \citep{zhang2024exact} instead derives an exact-inversion two-step recursion by combining a forward DDIM increment with a backward DDIM increment. Let
\(\Delta(i\!\to\!j\,|\,x_i)\) denote the DDIM increment from index \(i\) to \(j\) computed from \(x_i\) (i.e., \(x_j = x_i + \Delta(i\!\to\!j\,|\,x_i)\)).
Then BDIA computes \(x_{i-1}\) from \(x_{i+1}\) and \(x_i\) as
\begin{align*}
x_{i-1}
=\,
&x_{i+1}
-
(1-\gamma)(x_{i+1}-x_i) \\
&- \gamma\,\Delta(i\!\to\!i+1\,|\,x_i)
+
\Delta(i\!\to\!i-1\,|\,x_i),
\end{align*}
for a mixing parameter \(\gamma\in[0,1]\).

\textbf{BELM} \citep{wang2024belm} unifies these designs by showing that EDICT and BDIA fall into a class of \emph{bidirectional explicit linear multi-step} methods for \eqref{eq:diffusion_ivp_barsigma}. In particular, a two-step BELM sampler can be written as
\begin{equation}
\bar x_{i-1}
=
a_{i,2}\,\bar x_{i+1}
+
a_{i,1}\,\bar x_i
+
b_{i,1}\,h_i\,\bar\epsilon_\theta(\bar x_i,\bar\sigma_i,c).
\label{eq:belm_2step}
\end{equation}
The exact inversion property follows from the fact that the same linear relation can be solved in the opposite direction to recover \(\bar x_{i+1}\).
The authors derive an optimal choice of coefficients \(a_{i,1}=\frac{h^2_{i+1}-h^2_i}{h^2_{i+1}}\), \(a_{i,2}=\frac{h^2_i}{h^2_{i+1}}\), and \(b_{i,1}=-\frac{h_i+h_{i+1}}{h_{i+1}}\), forming the O-BELM variant of the scheme \citep[Eq.~(18)]{wang2024belm}.

\textbf{Rex} \citep{blasingame2025rex} combines the Lawson method with the
McCallum--Foster reversible solver for diffusion ODE and SDE sampling. Here, we
focus on the ODE variant in the noise-prediction parameterisation. Rex uses the
reparameterised time variable \(\tau=\sigma/\alpha\). We denote the corresponding
positive- and negative-step Princeps increments, induced by a chosen explicit
Runge--Kutta method, by \(\Psi_{h_i}\) and \(\Psi_{-h_i}\), where
\(h_i=\tau_{i-1}-\tau_i\). Mapping back to the original diffusion state gives
\begin{equation}
\begin{aligned}
x_{i-1}
&=
\frac{\alpha_{i-1}}{\alpha_i}
\bigl(\zeta x_i + (1-\zeta)\hat x_i\bigr)
+
\alpha_{i-1}\Psi_{h_i}(\tau_i,\hat x_i),
\\
\hat x_{i-1}
&=
\frac{\alpha_{i-1}}{\alpha_i}\hat x_i
-
\alpha_{i-1}\Psi_{-h_i}(\tau_{i-1},x_{i-1}).
\end{aligned}
\label{eq:rex_core}
\end{equation}
Here \(\zeta\) is a coupling parameter.

\subsection{Alternative ODE Parametrisations}
\label{sec:alternative_ode_parametrisations}
\paragraph{Half-logSNR parametrisation (\(\epsilon\)-prediction).}
Let \(\lambda := \log(\alpha_t/\sigma_t)\), which is strictly decreasing and invertible. For VP schedules one can show that \(\frac{d\log \alpha_\lambda}{d\lambda}=\sigma_\lambda^2\), and the probability-flow ODE can be written as
\begin{equation}
\frac{dx_\lambda}{d\lambda}
=
\sigma_\lambda^2\,x_\lambda-\sigma_\lambda\,\epsilon_\theta(x_\lambda,\lambda).
\label{eq:ode_lambda_eps_main}
\end{equation}
This is the \(\lambda\)-domain ODE used by DPM \citep{lu2022dpm}.

\paragraph{Half-logSNR parametrisation (\(x_0\)-prediction).}
DPM-Solver++ \citep{lu2025dpm} solves the equivalent ODE expressed in terms of the data predictor
\begin{equation}
    x_\theta(x_t,t):=\frac{x_t-\sigma_t\epsilon_\theta(x_t,t)}{\alpha_t}.
    \label{eq:x0_predictor}
\end{equation}
Substituting \(\sigma_\lambda\epsilon_\theta = x_\lambda-\alpha_\lambda x_\theta\) into \eqref{eq:ode_lambda_eps_main} yields
\begin{equation}
\frac{dx_\lambda}{d\lambda}
=
-\alpha_\lambda^2\,x_\lambda+\alpha_\lambda\,x_\theta(x_\lambda,\lambda).
\label{eq:ode_lambda_x0_main}
\end{equation}

\section{EES Solvers for the Diffusion ODE}
\label{sec:ees_solvers_for_diffusion}
Building on the reversible and near-reversible solver families reviewed in \cref{sec:reversible_solvers}, we now specialise EES schemes \citep{shmelev2025explicit} to the diffusion probability-flow ODE. We refer to \cref{appendix:ees_schemes} for the scheme definitions and Butcher tableaux. Here we state the forward and backward update formulas.

We work in the half-logSNR time variable \(\lambda_t := \log(\alpha_t/\sigma_t)\) and use the \(x_0\)-predictor from \eqref{eq:x0_predictor}. As discussed in \cref{sec:alternative_ode_parametrisations}, the probability-flow ODE can be written in \(\lambda\)-time as \eqref{eq:ode_lambda_x0_main}. Let \(\lambda_0>\lambda_1>\cdots>\lambda_N\) be a uniform grid with step \(h:=\lambda_{i-1}-\lambda_i>0\).
Following \citet{lu2025dpm}, we introduce the rescaled state \(y_\lambda := x_\lambda/\sigma_\lambda\), for which the dynamics become
\begin{equation}
\frac{d y_\lambda}{d\lambda}
= e^{\lambda}\,x_\theta(\sigma_\lambda y_\lambda,\lambda),
\label{eq:y_ode}
\end{equation}
using \(\alpha_\lambda/\sigma_\lambda = e^{\lambda}\) (see \cref{appendix:ode_parametrisation}).
Let \(\Phi_h(y,\lambda)\) denote the EES Runge--Kutta increment for \eqref{eq:y_ode} with step size \(h\), so that one step takes the form \(y \mapsto y + h\,\Phi_h(y,\lambda)\).

A forward step from \(\lambda_i\) to \(\lambda_{i-1}=\lambda_i+h\) is
\begin{equation}
\begin{aligned}
y_i &= \frac{x_i}{\sigma_{\lambda_i}}, \qquad
y_{i-1} = y_i + h\,\Phi_h(y_i,\lambda_i), \\
x_{i-1} &= \sigma_{\lambda_{i-1}}\,y_{i-1}.
\end{aligned}
\label{eq:ees_sampler_update_fwd}
\end{equation}
The corresponding backward step from \(\lambda_{i-1}\) to \(\lambda_i=\lambda_{i-1}-h\) is
\begin{equation}
\begin{aligned}
\tilde y_{i-1} &= \frac{x_{i-1}}{\sigma_{\lambda_{i-1}}}, \qquad
\tilde y_i = \tilde y_{i-1} - h\,\Phi_{-h}(\tilde y_{i-1},\lambda_{i-1}), \\
\tilde x_i &= \sigma_{\lambda_i}\,\tilde y_i.
\end{aligned}
\label{eq:ees_sampler_update_bwd}
\end{equation}
Because the endpoint rescaling is exact and EES methods are near symmetric, one forward step followed by the corresponding backward step satisfies \(\tilde x_i = x_i + \mathcal{O}(h^{m+1})\), where \(m\) depends on the chosen EES scheme. Finally, while the one-step EES construction itself is not tied to a particular ODE parametrisation, the overall editing performance can be sensitive to this choice. We therefore repeat the experiments from \cref{sec:experiments_editing} under alternative formulations in \cref{appendix:ode_formulation_ablations}.

\subsection{Dependence of EES on Trajectory Regularity}
\label{sec:dependence_of_EES_on_trajectory_regularity}

We briefly note the dependence of the error of a near-reversible method on the solution trajectory and its derivatives. Given a numerical method of order $n$, it is clear from the definition of order and the Taylor series expansion of the solution that the leading term of the local error takes the form $C y_t^{(n+1)} h^{n+1}$ for some constant $C$ independent of $y$ and $h$. In a similar fashion, the local \say{reversibility} error of a reversible solver (that is, the error in recovering the initial condition when the solver is run forwards and then backwards) has a leading term of the form $Cy_t^{(m+1)}h^{m+1}$. As such, regularity of the solution trajectory $y_t$ and its derivatives plays a significant role in the effectiveness of near-reversible methods such as EES schemes. To mitigate this effect, we apply a number of vector-field smoothing and guidance tricks (see \cref{sec:vector_field_smoothness}).

\subsection{Additional Reversible Solvers}
\label{sec:extra_solver_rev_heun}
For completeness, we also include Reversible Heun \citep{kidger2021efficient}
and McCallum--Foster \citep{mccallum2024efficient} in our experiments as
additional reversible solvers for the diffusion ODE. We use the same
probability-flow ODE parametrisation as in
\cref{sec:ees_solvers_for_diffusion}, but, based on the ablation in
\cref{appendix:ode_formulation_ablations}, we integrate it directly as a
black-box vector field rather than using the rescaled form \eqref{eq:y_ode}. We
discretise on a uniform \(t\)-grid \(t_0,t_1,\ldots,t_N\) and use the
corresponding \(\lambda_{t_i}\) values. See \cref{appendix:reversible_heun} and
\cref{appendix:mccallum_foster} for the exact forms. In
\cref{sec:experiments}, we find that although these methods originate from the
neural differential equations literature, they are competitive with the other
diffusion samplers in our setting.

\section{Stability}
\label{sec:stability}

In this section, we show that EES methods offer significantly stronger stability guarantees than multi-step methods such as BELM, due to their near-symmetric properties and Runge--Kutta type.

\subsection{Multi-Step Methods}

It is well known that multi-step methods typically possess lower stability than one-step methods. A concrete example of this reduced stability is given by the second Dahlquist barrier \cite{dahlquist1963special}, which famously states that no multi-step method of order greater than 2 can be A-stable.\par

The stability of multi-step methods can be difficult to analyse, especially in the case of time-varying coefficients such as the ones used by the BELM sampler \eqref{eq:belm_2step}. A common proxy for the stability of a linear method is \emph{zero-stability}, which requires that small perturbations to the initial condition lead to small changes in the numerical solution. An equivalent definition which is aligned closer with the notion of linear-stability of one-step methods is given by the following.

\begin{definition}[Zero-Stability]
    Given a linear multistep method, let $P$ denote the characteristic polynomial which arises from the application of the method to the equation $y'=0$. The method is said to be \emph{zero-stable} if all roots $\zeta$ of $P$ satisfy the root condition $|\zeta| \leq 1$.
\end{definition}

As discussed in \citet{wang2024belm}, both DDIM and BELM are zero-stable, whilst the zero-stability of the EDICT and BDIA solvers is unclear. Whilst zero-stability is a useful metric for multi-step solvers, the condition is significantly weaker than that of the linear stability of one-step methods. As we will see, the stability guarantees of EES schemes far exceed those of zero-stability.

\subsection{One-Step Methods}

We begin by defining the more general notion of \emph{linear stability}.

\begin{definition}
    Consider a one-step ODE method $y_{n+1} = y_n + h\Phi_h(y_n)$ applied to the linear test ODE $y' = \lambda y$, $y_0 = 1$ for $\lambda \in \mathbb{C}$. Let $z = \lambda h$. The linear stability domain of the method $\Phi$ is defined by
    \(
        \mathcal{D} = \{z \in \mathbb{C} : \lim_{n \to \infty}y_n = 0\}.
    \)
\end{definition}

We note that linear stability is a more expressive notion of stability than zero-stability, as the latter corresponds to the case $z = 0$. 


\begin{theorem}{\cite{kidger2021efficient}}\label{thm:rev_heun_stability}
    When applied to the linear test ODE $y'=\lambda y$, the Reversible Heun
    method produces a bounded numerical sequence if and only if
    $\lambda h \in [-i,i]$.
\end{theorem}

As remarked in \citet{kidger2021efficient}, this bounded-stability interval is also the absolute stability region for the reversible asynchronous leapfrog integrator \citep{zhuang2021mali}. Since $[-i,i]$ has empty interior in
$\mathbb{C}$, these methods are \textit{nowhere linearly stable}. \citet{blasingame2025rex} further show that this is also the case for BDIA and fixed-step O-BELM. In contrast, the McCallum--Foster method has a non-trivial linear
stability region, albeit smaller than that of the base solver \citep[Theorem~2.3]{mccallum2024efficient}. We restate this result in \cref{appendix:mccallum_linear_stability}. Rex inherits its non-zero stability guarantee from McCallum--Foster \citep[Section~3.3 and Appendix~A.2]{blasingame2025rex}.

By forgoing exact reversibility, $\mathrm{EES}$ methods are able to offer significantly larger stability domains, even when compared to classical Runge--Kutta methods. The stability domains of $\mathrm{EES}(2,5)$ and $\mathrm{EES}(2,7)$ are shown in Figure \ref{fig:ees_stability}, and their exact forms are given in Appendix \ref{appendix:ees_stability}. $\mathrm{EES}(2,5)$ possesses a stability domain comparable to that of RK4 whilst requiring one fewer stage. $\mathrm{EES}(2,7)$ is considerably more stable than RK5 whilst requiring 2 fewer stages.

\section{Experiments}
\label{sec:experiments}

In \cref{sec:vector_field_smoothness}, we first isolate the role of trajectory regularity. The reconstruction experiments show that the EES reconstruction error grows with vector-field irregularity (i.e. high guidance scale), and that smoothing the vector field substantially reduces this error. This is consistent with the dependence of reversibility error on trajectory regularity discussed in \cref{sec:dependence_of_EES_on_trajectory_regularity}. In the editing experiments, smoothing also improves background preservation for EES and yields gains for several other solvers.
	
We then evaluate the editing performance of each solver on the \textit{small edits} PIE-Bench dataset in \cref{sec:experiments_editing}. \cref{fig:metrics_trade_off} and \cref{tab:solver_metrics} show the trade-off between background preservation and alignment with the intended edit. The optimal solver choice therefore depends on the user's priorities, and the qualitative examples in \cref{fig:example_figure} make this trade-off visible.
	
Finally, in \cref{sec:large_edits_and_stability}, we test the solvers' robustness under \textit{large edits}. \cref{tab:clip_whole_sd_vs_smooth} shows that solvers with smaller stability regions visibly degrade in edit quality when the sampling trajectory deviates substantially from the inverted path, whereas EES remains the most reliable of all (near) reversible solvers. Additional results and ablations can be found in \cref{appendix:additional_results}.

\textbf{Image editing setup.}
We use Stable Diffusion v1.5 \citep{Rombach_2022_CVPR} as the base model in all experiments. To compare methods under matched compute, we fix a budget of 48 model evaluations per sample and adjust the number of inference steps
accordingly: 48 steps for BDIA, BELM, and DDIM; 24 steps for EDICT and Reversible Heun; 16 steps for EES(2,5); and 12 steps for EES(2,7). Rex and McCallum--Foster require twice as many model evaluations as their base
solvers, making higher-order or adaptive-step variants too expensive for our 48-evaluation setup. Hence, we choose to benchmark against Rex/McCallum--Foster (Euler) and Rex/McCallum--Foster (Midpoint), using 24 and 12 steps, respectively. We use the standard inversion strength \(s=0.8\), as defined by EDICT, for the editing
experiments in \cref{sec:experiments_smooth_diffusion} and \cref{sec:experiments_editing}, and \(s=1\) in the stability study in \cref{sec:large_edits_and_stability}, where stronger edits are required. Following previous work, we fix the editing guidance scale at \(g=3.0\). We keep the recommended hyperparameters for EDICT (\(p=0.93\)), BDIA
\((\gamma=0.96)\), and Rex/McCallum--Foster (\(\zeta=0.999\)).

\textbf{Editing benchmark and metrics.}
We evaluate on PIE-Bench \citep{ju2023direct}, which contains 700 images spanning 10 edit types and provides a human-annotated editing mask for each example, together with source and target prompts. We report four metrics grouped into two aspects: (i) background preservation, measured outside the editing mask using LPIPS \citep{zhang2018unreasonable}; (ii) prompt–image consistency, measured by CLIP Similarity \citep{radford2021learning} on the masked region, PickScore \citep{kirstain2023pick}, and ImageReward \citep{xu2023imagereward}. In extended results in \cref{appendix:additional_results}, we also report PSNR as another background metric, and a CLIP Similarity on the whole image. Full implementation details are deferred to \cref{appendix:experiments}.

\subsection{The Impact of Vector-Field Smoothness}
\label{sec:vector_field_smoothness}
We first study image reconstruction to isolate inversion error in EES schemes, and show that, as indicated in \cref{sec:dependence_of_EES_on_trajectory_regularity}, swapping to a model version fine-tuned to produce smoother trajectories (Smooth Diffusion; \citealt{guo2024smooth}) substantially improves EES reversibility, which translates to better background preservation in image editing.

\subsubsection{Image reconstruction}
\label{sec:experiments_reconstruction}

In image reconstruction, given an image–caption pair, we invert the image to obtain a latent \(x_T\) and then reconstruct by running the solver back to \(x_0\) with the same caption. We report pixel-space mean squared error (MSE) computed on pixels normalised to  \([-1,1]\), averaged over 100 images from the MS COCO 2014 \citep{lin2014microsoft} Validation set. To match the image-editing setup, we fix a budget of 48 model evaluations per sample for all solvers.

Supporting \citet[Section D]{wallace2023edict}, which observes that classifier-free guidance can induce vector-field roughness, \cref{fig:reconstruction_vs_guidance} shows that the EES schemes incur increasing reconstruction error as guidance grows, whereas the exactly reversible solvers remain largely unaffected. Repeating the experiment with the Smooth Diffusion \citep{guo2024smooth} fine-tuned Stable Diffusion 1.5 checkpoint substantially reduces this guidance sensitivity, suggesting that sufficiently smooth trajectories lead to reliable EES inversion.

\begin{figure}[t]
    \centering
    \includegraphics[width=\columnwidth]{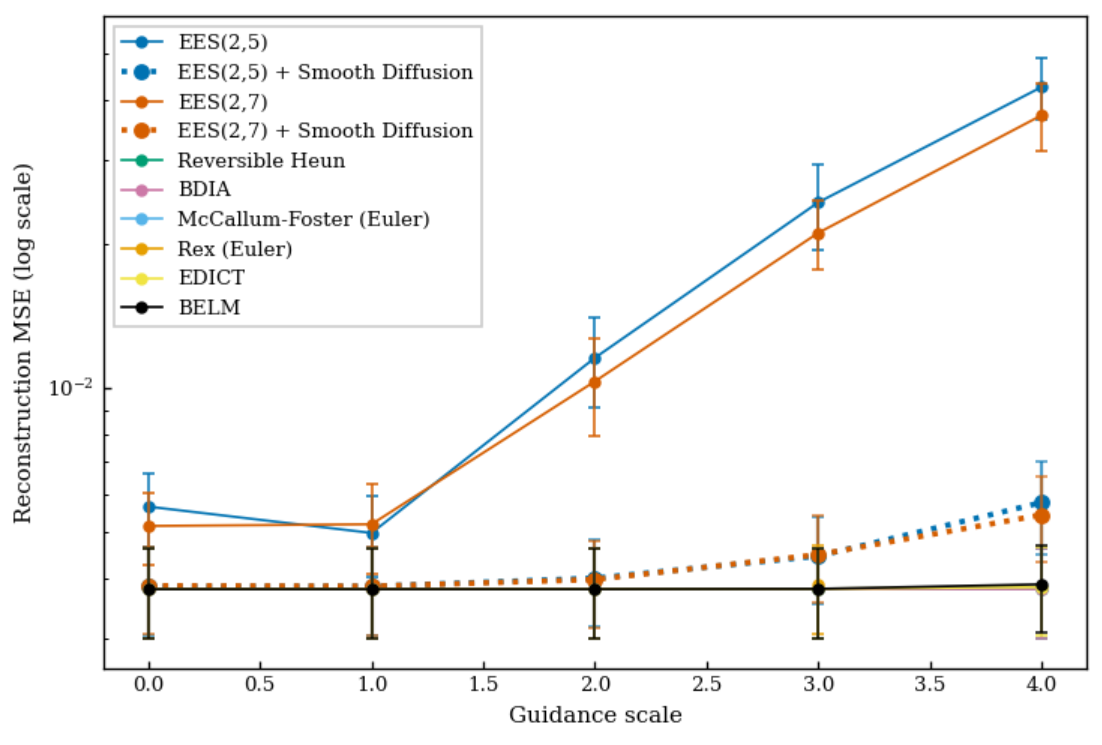}
    \caption{Reconstruction MSE of (near) reversible solvers over 100 MS COCO 2014 images. EES schemes degrade with increasing guidance, but Smooth Diffusion greatly reduces this effect. The error bars represent 95\% confidence intervals.}
    \label{fig:reconstruction_vs_guidance}
\end{figure}

\subsubsection{Image editing}
\label{sec:experiments_smooth_diffusion}

We now show that the importance of trajectory smoothness for EES inversion observed in \cref{sec:experiments_reconstruction} translates to image editing quality. In \cref{fig:grid_retriever}, we compare EES under Stable Diffusion 1.5 and its Smooth Diffusion \citep{guo2024smooth} fine-tuned checkpoint. Using Smooth Diffusion largely preserves the background while still applying the intended edit. For quantitative evidence, we report full PIE-Bench results for both checkpoints across all solvers in \cref{tab:solver_metrics_PIEBench_all}, and we also evaluate alternative smoothing heuristics that do not require changing the model, namely NPI and Proximal Guidance. Part 1 of \cref{fig:metrics_tricks} visualises how Smooth Diffusion shifts each solver’s editing metrics. Since Smooth Diffusion does not benefit every solver equally, in subsequent sections we report, for each solver, the best-performing variant.

\begin{figure}[t]
    \centering
    \includegraphics[width=\columnwidth]{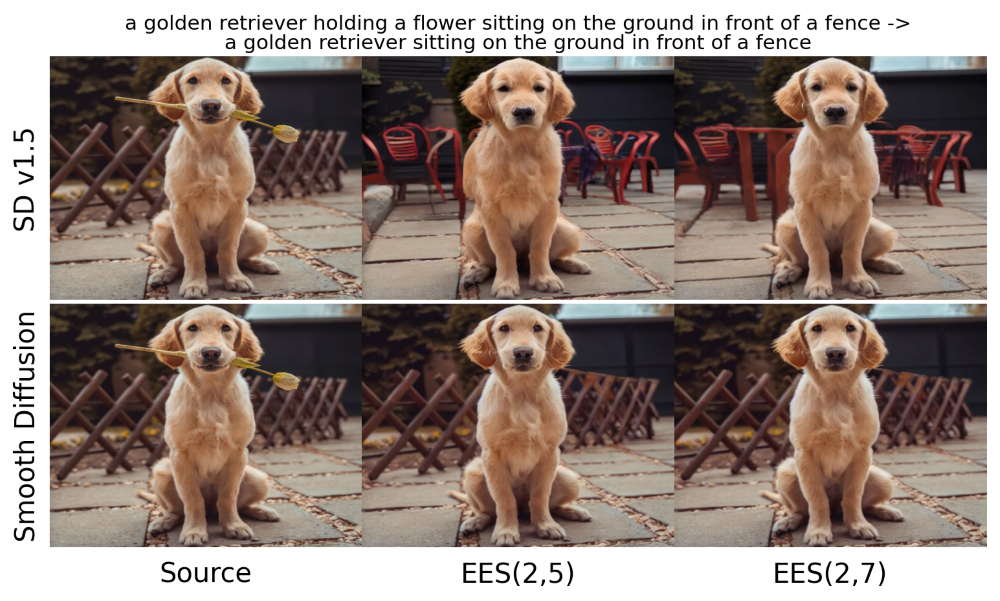}
    \caption{Qualitative editing comparison for EES under Stable Diffusion 1.5 (top) and Smooth Diffusion (bottom). Smooth Diffusion improves background preservation (e.g., the fence) while maintaining the intended edit. See \cref{fig:edit_grid_retriever_extended} for an extended grid including all solvers.}
    \label{fig:grid_retriever}
\end{figure}

\subsection{Small Edits: Performance on PIE-Bench}
\label{sec:experiments_editing}

We compare all solvers on PIE-Bench using the metrics in \cref{tab:solver_metrics}. \cref{fig:metrics_trade_off} summarises the main trade-off: improving prompt adherence tends to worsen background preservation. Exactly reversible solvers (e.g., EDICT) sit at the background-preserving end, while DDIM achieves strong prompt alignment but substantially degrades backgrounds. EES methods move the reversible frontier towards stronger edits, improving masked-region CLIP similarity while keeping background degradation substantially below DDIM.

We illustrate these trends qualitatively in \cref{fig:edit_grid_storm_trooper}, where EES preserves the background better than DDIM, and in row 1, \cref{fig:example_figure_small}, where EES achieves visibly stronger edit fidelity than the other reversible solvers, with extended versions of these figures in Figures \ref{fig:edit_grid_storm_trooper_extended} and \ref{fig:edit_grid_owl_extended}.

\begin{table}[t]
\caption{\textit{Small edits} (PIE-Bench) image editing evaluation results. We see the trade-off between background preservation and edit quality visualised in \cref{fig:metrics_trade_off}. Here, \textsc{MCF} denotes McCallum-Foster and ``+ Smooth'' denotes ``+ Smooth Diffusion''. (\textbf{best}, \underline{second-best}, $^{\star}$ third-best among (near) reversible solvers.)}
\label{tab:solver_metrics}
\vskip 0.10in
\centering
\scriptsize
\renewcommand{\arraystretch}{1.12}
\setlength{\tabcolsep}{3.2pt}
\resizebox{\columnwidth}{!}{%
\begin{tabular}{l|c|ccc}
\hline\hline
\multicolumn{1}{c|}{\textbf{Solver}} &
\multicolumn{1}{c|}{\textbf{Background Quality}} &
\multicolumn{3}{c}{\textbf{Edit Quality}} \\
\multicolumn{1}{c|}{} &
\textbf{LPIPS}$\times 10^3\,\downarrow$ &
\textbf{CLIP (Edited)}$\,\uparrow$ &
\textbf{PickScore}$\,\uparrow$ &
\textbf{ImageReward}$\,\uparrow$ \\
\hline

\multicolumn{5}{l}{\textbf{(Near) reversible solvers}} \\
\hline
EDICT + Smooth
& 45.21$^{\star}$ & 21.26 & \textbf{21.34} & 0.08 \\

Reversible Heun
& 52.45 & 21.41 & 21.18 & 0.05 \\

BDIA
& 50.52 & 21.41 & 21.21 & 0.04 \\

BELM + Smooth
& 75.24 & 21.28 & 20.97 & 0.02 \\

MCF (Euler)
& \textbf{36.80} & 20.50 & 21.17 & -0.17 \\

MCF (Midpoint) + Smooth
& \underline{41.37} & 20.68 & 21.20 & -0.09 \\

Rex (Midpoint) + Smooth
& 67.88 & \underline{21.62} & 21.15 & \underline{0.11} \\

Rex (Euler) + Smooth
& 60.31 & 20.64 & 20.87 & -0.18 \\

EES(2,5) + Smooth
& 59.30 & 21.54$^{\star}$ & \underline{21.30} & 0.11$^{\star}$ \\

EES(2,7) + Smooth
& 62.52 & \textbf{21.64} & 21.30$^{\star}$ & \textbf{0.12} \\

\hline\hline
\end{tabular}%
}
\vskip -0.10in
\end{table}

\begin{figure}[t]
    \centering
    \includegraphics[width=\columnwidth]{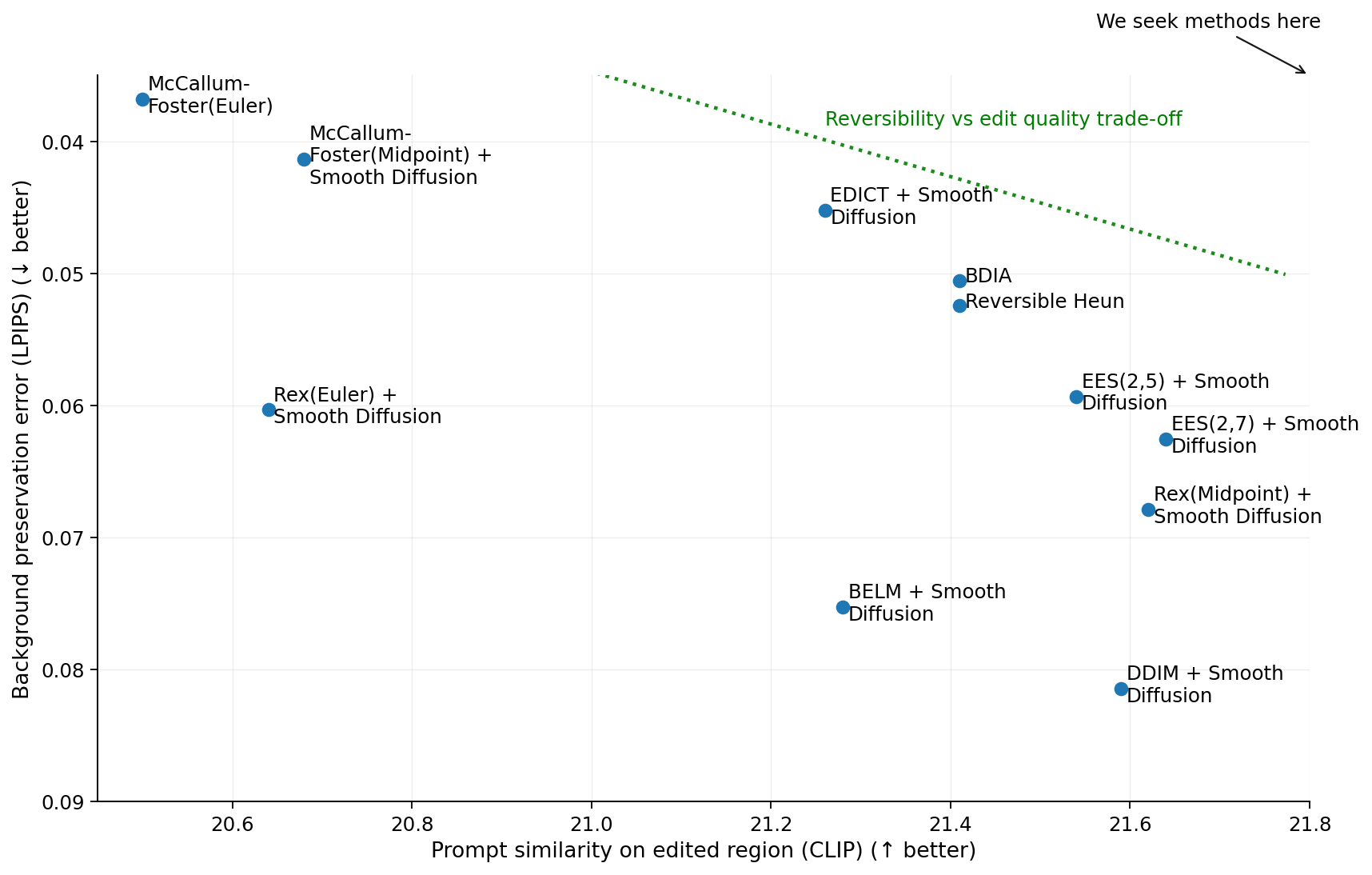}
    \caption{Prompt alignment (CLIP similarity of the edited region) versus background preservation (LPIPS outside the edit mask) on PIE-Bench. Reversible and near-reversible solvers lie on the trade-off curve between background similarity and edit quality.}
    \label{fig:metrics_trade_off}
\end{figure}

\begin{figure}[t]
    \centering
    \includegraphics[width=\columnwidth]{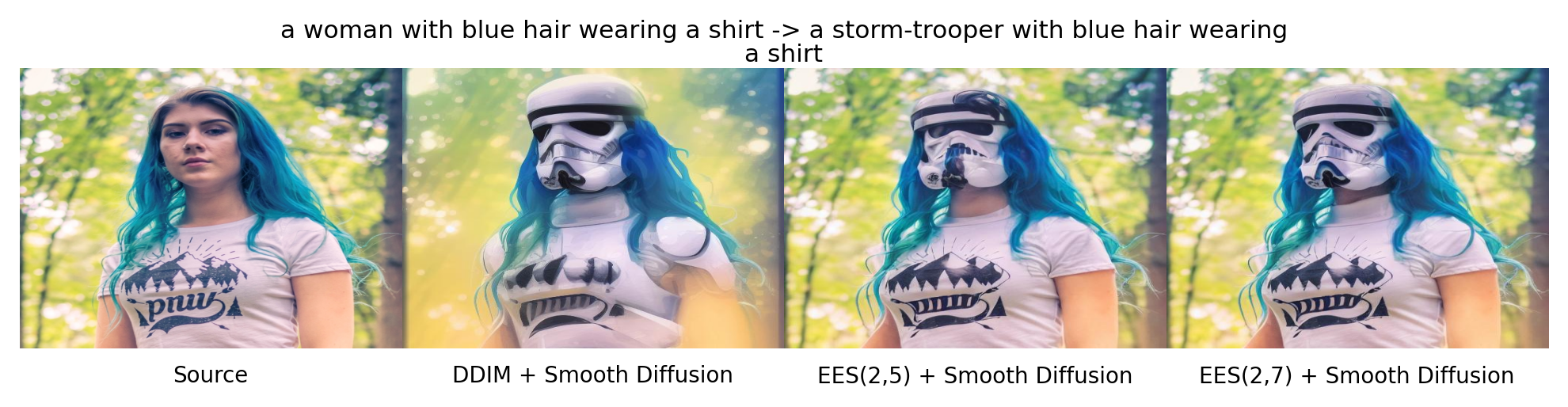}
    \caption{Example comparing DDIM and EES under the same edit. DDIM is worse at preserving background. See \cref{fig:edit_grid_storm_trooper_extended} for an extended version with more solvers.}
    \label{fig:edit_grid_storm_trooper}
\end{figure}

\subsection{Large Edits and Stability}
\label{sec:large_edits_and_stability}

Here we highlight the main advantage of EES schemes over other reversible solvers: stability. We evaluate all solvers in regimes that induce large deviations from the original trajectory, and find that most exactly reversible solvers proposed in recent literature degrade substantially in edit quality, whereas EES remains stable and achieves the strongest performance across all metrics.

We construct a large-prompt-deviations experiment from the \textit{random images} category of PIE-Bench (140 images). For each image, we shift the target prompt by one position in the dataset, so the edit prompt is typically unrelated to the source image and may require a large semantic change. Since the original edit masks are no longer meaningful, we evaluate prompt alignment using CLIP similarity on the full image as our main metric, alongside PickScore and ImageReward. Here, we report all results using the Smooth Diffusion checkpoint. 

\Cref{tab:clip_whole_sd_vs_smooth} shows that EES attains the strongest performance across metrics among the (near) reversible solvers under these large prompt deviations. Many exactly reversible solvers produce implausible outputs in this setting, failing to satisfy the new prompt condition when the required trajectory change is large. Rex (Midpoint) also does notably well, which is consistent with its non-trivial stability region. \Cref{fig:example_figure_large} provides representative examples, with additional solvers in \cref{fig:large_edits_retriever_extended}. Results for the base model and for the other smoothing methods are provided in \cref{tab:clip_whole_by_setting}.

\begin{table}[t]
\caption{Image editing metrics with Smooth Diffusion on \textit{large edits}. We no longer have a well-defined background, so we report edit quality metrics with the CLIP score on the whole image. We see that EES remains most robust under challenging edits. See \cref{tab:clip_whole_by_setting} for extended results with all smoothing methods. \;(\textbf{best}, \underline{second-best}, $^{\star}$ third-best among (near) reversible solvers)}
\label{tab:clip_whole_sd_vs_smooth}
\vskip 0.10in
\centering
\scriptsize
\renewcommand{\arraystretch}{1.15}
\setlength{\tabcolsep}{4.0pt}
\resizebox{\columnwidth}{!}{%
\begin{tabular}{l|ccc}
\hline\hline
\multicolumn{1}{c|}{\textbf{Solver}} &
\textbf{CLIP} $\uparrow$ &
\textbf{PickScore} $\uparrow$ &
\textbf{ImageReward} $\uparrow$ \\
\hline

\multicolumn{4}{l}{\textbf{Reversible solvers}} \\
\hline
EDICT                    & 21.88 & 19.58 & -0.97 \\
Reversible Heun          & 20.12 & 19.20 & -1.29 \\
BDIA                     & 20.77 & 19.27 & -1.22 \\
BELM                     & 22.19 & 19.36 & -0.94 \\
McCallum-Foster (Euler)    & 11.49 & 17.49 & -2.16 \\
McCallum-Foster (Midpoint) & 16.05 & 18.25 & -1.97 \\
Rex (Midpoint)             & 22.77$^{\star}$ & 19.68$^{\star}$ & -0.85$^{\star}$ \\
Rex (Euler)                & 15.17 & 17.50 & -2.14 \\
EES(2,5)                 & \underline{23.63} & \textbf{20.09} & \textbf{-0.57} \\
EES(2,7)                 & \textbf{23.66} & \textbf{20.09} & \underline{-0.58} \\
\hline\hline
\end{tabular}%
}
\vskip -0.10in
\end{table}

\section{Conclusion and Limitations}
We study reversible diffusion ODE solvers for inversion-based editing through the lens of editing stability, and show that exact reversibility alone is not enough: when an edit induces a large deviation from the inversion trajectory, existing reversible schemes can become numerically unstable and suffer sharp drops in quality. We address this with near-reversible EES solvers, which trade exact reversibility for substantially improved numerical stability. Empirically, EES is competitive with exactly reversible solvers on standard PIE-Bench edits and is the best-performing (near) reversible family in our study under large edits.

While solver design is a crucial building block, making (near) reversible solvers production-quality image editors will require integrating them into more powerful pipelines. We aim to move reversible solvers closer to this goal by connecting them with existing image editing work: incorporating smoothing tricks, exploring ODE parametrisations, and benchmarking on PIE-Bench. Many directions remain open for future work, such as combining EES with attention-preservation methods like Prompt-to-Prompt \citep{hertz2022prompt} or MasaCtrl \citep{cao2023masactrl}  and extending evaluation beyond inversion-based methods (e.g., \citet{zhang2023adding, lugmayr2022repaint, tumanyan2023plug}).

\newpage
\section*{Impact Statement}
This paper presents work whose goal is to advance the field of machine learning. There are many potential societal consequences of our work, none of which we feel must be specifically highlighted here.
\section*{Acknowledgments}
Barbora Barancikova is supported by UK Research and Innovation [UKRI Centre for Doctoral
Training in AI for Healthcare grant number EP/S023283/1].

\bibliography{references}
\bibliographystyle{icml2026}

\newpage
\appendix
\onecolumn

\section{EES Schemes}
\label{appendix:ees_schemes}

In this section, we provide a brief account of Explicit and Effectively Symmetric (EES) Runge--Kutta schemes \cite{shmelev2025explicit} and provide the general forms for $\mathrm{EES}(2,5)$ and $\mathrm{EES}(2,7)$ methods.

\subsection{Runge--Kutta Schemes}

Consider solving an ODE $y' = f(t,y)$ numerically using a one-step method 
\begin{equation}
\label{1-stepmethod}
    y_{n+1} = y_n + h \Psi_h(t_n, y_n).
\end{equation}
The method $\Psi_h$ is an $s$-stage Runge--Kutta (RK) method if $\Psi$ can be written in the form
\begin{align*}
    y_{n+1} &= y_n + h \sum_{i=1}^s b_i k_i,\\
    k_i &= f\left(t_n + c_i h, \, y_n + h\sum_{j=1}^{s} a_{ij} k_j\right), \quad i=1,\ldots,s,
\end{align*}
with RK matrix $A = (a_{ij})_{1 \leq i,j \leq s}$, weights $b = (b_i)_{1 \leq i \leq s}$ and nodes $c = (c_i)_{1 \leq i \leq s}$. These coefficients are typically collected in the \emph{Butcher tableau} of the scheme
\[
\begin{array}{c|ccccc}
c_1 &a_{11} & a_{12} &\cdots& a_{1s}&\\
c_2 & a_{21} & a_{22}&\cdots&a_{2s}&\\
\vdots & \vdots & \vdots & \ddots & \vdots\\
c_s & a_{s1} & a_{s2} & \cdots & a_{s, s} &\\
\hline
& b_1 & b_2 & \cdots & b_s
\end{array}
\]

In the case of an explicit scheme, the coefficients must satisfy $a_{ij} = 0$ for all $i \leq j$. In this case, the Butcher tableau takes the form

\[
\begin{array}{c|cccccc}
0 &&&&&&\\
c_2 & a_{21} &&&&&\\
c_3 & a_{31} & a_{32} &&&&\\
c_4 & a_{41} & a_{42} & a_{43} &&&\\
\vdots & \vdots & \vdots & \vdots & \ddots &\\
c_s & a_{s1} & a_{s2} & a_{s3} & \cdots & a_{s, s-1} &\\
\hline
& b_1 & b_2 & b_3 & \cdots & b_{s-1} & b_s
\end{array}
\]

\subsection{Explicit and Effectively Symmetric (EES) Schemes}

\begin{definition}
    A one-step Runge--Kutta method with update rule $y_1 = y_0 + h \Phi_h(y_0)$ is said to be of order $n$ if
    \begin{equation}
        |y_1 - y_t| = \mathcal{O}(h^{n+1})
    \end{equation}
    and of \emph{anti-symmetric} order $m$ if
    \begin{equation}
        |\backvec{y}_1 - y_0| = \mathcal{O}(h^{m+1}),
    \end{equation}
    where $\backvec{y}_1 = y_1 - h \Phi_{-h}(y_1)$. An explicit Runge--Kutta method is said to belong to the class of $\mathrm{EES}(n,m)$ schemes if it is of order $n$ and anti-symmetric order $m$.
\end{definition}

\begin{proposition}
\label{prop:EES_2_5}
    \cite{shmelev2025explicit} 3-stage Runge--Kutta schemes belonging to $\mathrm{EES}(2,5)$ take the form:
    \begin{equation}
    {\renewcommand{\arraystretch}{2.2}
    \begin{array}{c|ccc}
    0 &&&\\
    \displaystyle\frac{1+2x}{4(1-x)} & \displaystyle\frac{1+2x}{4(1-x)} &&\\[7pt]
    \displaystyle\frac{3}{4(1-x)} & \displaystyle\frac{(4x-1)^2}{4(x-1)(1-4x^2)} & \displaystyle\frac{1-x}{(1-4x^2)} &\\[7pt]
    \hline
    & x & \displaystyle\frac{1}{2} & \displaystyle\frac{1}{2} - x
    \end{array}
    }
    \end{equation}
    for some $x \in \mathbb{R}$, $x \neq 1, \pm \frac{1}{2}$.
\end{proposition}

Following \citet{shmelev2025explicit}, we refer to the scheme with parameter $x$ as $\mathrm{EES}(2,5;x)$. In particular, $\mathrm{EES}(2,5;1/10)$ has the Butcher tableau:

\begin{equation}
    {\renewcommand{\arraystretch}{1.2}
    \begin{array}{c|ccc}
    0 &&&\\
    1/3 & 1/3 &&\\
    5/6 & -5/48 & 15/16 &\\
    \hline
    & 1/10 & 1/2 & 2/5
    \end{array}
    }
    \label{eq:butcher_ees_25}
\end{equation}

Following \citet{shmelev2025explicit}, we will refer to $\mathrm{EES}(2,5;1/10)$ as \emph{the} $\mathrm{EES}(2,5)$ method.

\begin{proposition}\label{prop:EES_2_7}
   \cite{shmelev2025explicit} For some $x \in \mathbb{R}$, 4-stage RK schemes in $\mathrm{EES}(2,7)$ take the form:
    \begin{align*}
        b_1 &= x,\\
        b_2 &= \frac{1}{2}(2 \mp \sqrt{2}) - (1 \mp\sqrt{2})x,  & \alpha &:= \frac{(2x \pm \sqrt{2})}{(2x - 1)(-2x \mp \sqrt{2} + 1)}\\
        b_3 &= (1 \mp\sqrt{2})(x-1),  & \beta &:= \frac{1}{(2x - 1)(1 \mp \sqrt{2}-2x)(2 \mp \sqrt{2} - 2x)}\\
        b_4 &= \frac{1}{2}(2 \mp \sqrt{2}) - x,
    \end{align*}
    \begin{align*}
        a_{21} &= \frac{-2\pm \sqrt{2} (1 - 2x)}{4(x-1)}\\
        a_{31} &= \frac{(2x \pm \sqrt{2} - 2)(4x \pm \sqrt{2} - 2)}{\pm 4\sqrt{2}(x - 1)}\, \alpha\\
        a_{32} &= \frac{1}{2}(-1 \pm\sqrt{2}) \,\alpha\\
        a_{41} &= \frac{(2x \mp\sqrt{2})(-40x^4 +(80 \mp 40\sqrt{2})x^3 -(88 \mp 60\sqrt{2})x^2 +(48 \mp 34\sqrt{2})x \pm 7\sqrt{2} - 10)}{4(x-1)(2x^2-1)} \, \beta\\
       a_{42} &= (2 \mp \sqrt{2}) x (x-1) (4x \pm \sqrt{2} - 2) \, \beta\\
       a_{43} &= \frac{(2 \mp\sqrt{2})(2x \mp\sqrt{2})(2 \pm \sqrt{2} - 2x)(x-1)(2x-1)}{4(2x^2-1)(2x^2 - 4x + 1)}
    \end{align*}
\end{proposition}

Following \citet{shmelev2025explicit}, we refer to the scheme with positive $\sqrt{2}$ and parameter $x$ as $\mathrm{EES}(2,7;x)$. In particular, $\mathrm{EES}(2,7;\frac{1}{14}(5-3\sqrt{2}))$ has the Butcher tableau:

\begin{equation}
    {\renewcommand{\arraystretch}{2}
    \begin{array}{c|cccc}
    0&&&\\
    \frac{1}{3}(2-\sqrt{2}) & \frac{1}{3}(2-\sqrt{2})&&&\\[7pt]
    \frac{1}{6}(2 + \sqrt{2}) & \frac{1}{24}(-4+\sqrt{2})  & \frac{1}{8}(4 + \sqrt{2})\displaystyle  &&\\[7pt]
    \frac{1}{6}(4 + \sqrt{2}) &\frac{1}{168}(-176+145\sqrt{2}) &\frac{3}{56}(8-5\sqrt{2})&\frac{3}{7}(3-\sqrt{2})&\\[7pt]
    \hline
    & \frac{1}{14}(5-3\sqrt{2}) & \frac{1}{14}(3 + \sqrt{2}) & \frac{3}{14}(-1+2\sqrt{2})   & \frac{1}{14}(9-4\sqrt{2})
    \end{array}
    }
    \label{eq:butcher_ees_27}
\end{equation}

Following \citet{shmelev2025explicit}, we will refer to $\mathrm{EES}(2,7;\frac{1}{14}(5-3\sqrt{2}))$ as \emph{the} $\mathrm{EES}(2,7)$ method.

\section{Stability}
\subsection{Linear Stability of EES Methods}\label{appendix:ees_stability}

\begin{theorem}{\citep{shmelev2025explicitneural}}\label{thm:ees25_ode_stability}
    Suppose that, for some $x \notin \{1, \pm \frac{1}{2}\}$, $\mathrm{EES}(2,5;x)$ is used to obtain a solution $\{y_n\}_{n \geq 0}$ to the linear test equation $dy = \lambda y\, dt$, where $\lambda \in \mathbb{C}$ and $y_0 \neq 0$. Then $y_n \to 0$ as $n \to \infty$ if and only if
    \begin{equation}
        \left\lvert 1 + z + \frac{1}{2} z^2 + \frac{1}{8} z^3 \right\rvert < 1,
    \end{equation}
    where $z = \lambda h$.
\end{theorem}
\begin{proof}
For an explicit $s$-stage Runge--Kutta method applied to $dy = \lambda y\, dt$, the numerical solution satisfies $y_{n+1} = R(z)\,y_n$ where $z = \lambda h$ and the stability polynomial is $R(z) = 1 + zb^T(I-zA)^{-1}e$.  Hence $y_n \to 0$ if and only if $|R(z)| < 1$. The result follows by substituting the general Butcher tableau from \cref{prop:EES_2_5}, giving $R(z) = 1 + z + \frac{1}{2}z^2 + \frac{1}{8}z^3$ for all choices of the $\mathrm{EES}$ parameter $x$.
\end{proof}

\begin{theorem}\label{thm:ees27_ode_stability}
    Suppose that, for some $x \notin \{ 1, \frac{1}{2}, \pm \frac{\sqrt{2}}{2}, 2 \pm \sqrt{3}, \frac{1}{2}(1-\sqrt{2}), \frac{1}{2}(2-\sqrt{2})\}$, $\mathrm{EES}(2,7;x)$ is used to obtain a solution $\{y_n\}_{n \geq 0}$ to the linear test equation $dy = \lambda y\, dt$, where $\lambda \in \mathbb{C}$ and $y_0 \neq 0$. Then $y_n \to 0$ as $n \to \infty$ if and only if
    \begin{equation}
        \left\lvert 1 + z + \frac{1}{2} z^2 + \frac{1}{4}(2 - \sqrt{2}) z^3 + \frac{1}{8}(3 - 2\sqrt{2})z^4 \right\rvert < 1,
    \end{equation}
    where $z = \lambda h$.
\end{theorem}
\begin{proof}
As in Theorem \ref{thm:ees25_ode_stability}, the result follows by substituting the general Butcher tableau from \cref{prop:EES_2_7}, giving the stability polynomial $R(z) = 1 + z + \frac{1}{2} z^2 + \frac{1}{4}(2 - \sqrt{2}) z^3 + \frac{1}{8}(3 - 2\sqrt{2})z^4$ for all choices of the $\mathrm{EES}$ parameter $x$.
\end{proof}

The stability domains of $\mathrm{EES}(2,5)$ and $\mathrm{EES}(2,7)$ are shown in Figure \ref{fig:ees_stability}.

\begin{figure}[ht]
    \centering

    \begin{subfigure}[tbp]{0.49\linewidth}
        \centering
        \includegraphics[width = \linewidth, trim={0.1cm 0.1cm 0.1cm 0.8cm},clip]{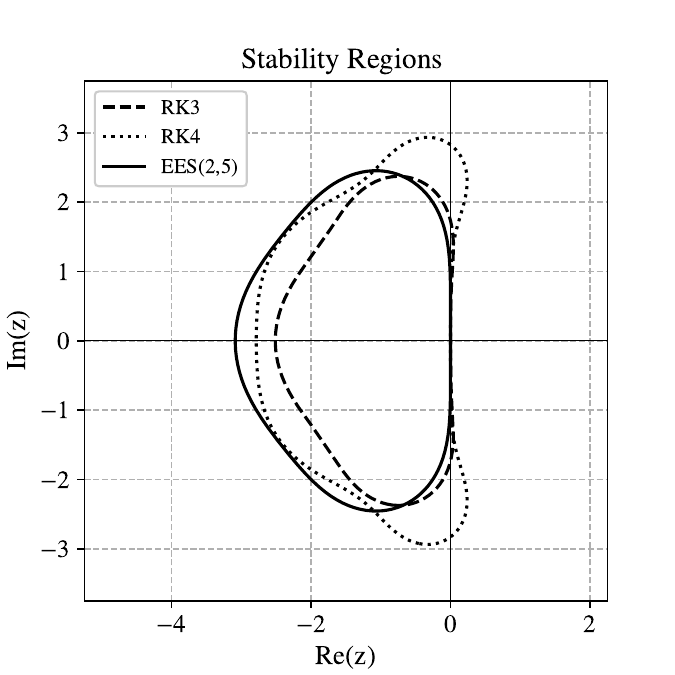}
    \end{subfigure}%
    ~ 
    \begin{subfigure}[tbp]{0.49\linewidth}
        \centering
        \includegraphics[width = \linewidth, trim={0.1cm 0.1cm 0.1cm 0.8cm},clip]{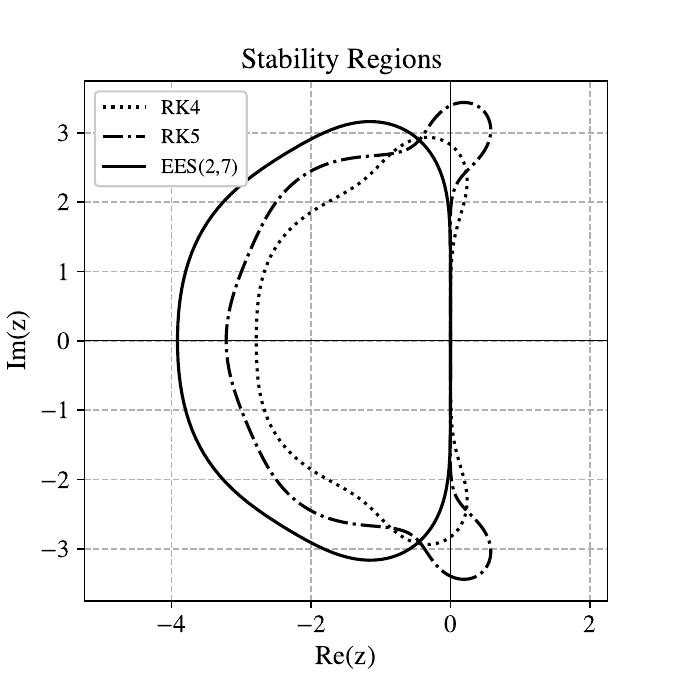}
    \end{subfigure}
    \caption{\citep[Fig. 2]{shmelev2025explicit} Stability domains for Kutta's RK3, RK4, Nystr{\"o}m's RK5, $EES(2,5;1/10)$ and $EES(2,7;\frac{1}{14}(5 - 3 \sqrt{2}))$.}
    \label{fig:ees_stability}
\end{figure}

\section{Reversible Heun}
\label{appendix:reversible_heun}
The Reversible Heun scheme introduced in \citet{kidger2021efficient} is an algebraically reversible analogue of the classical Heun Runge-Kutta scheme. The scheme is efficient, requiring only one function evaluation per step, but it suffers from particularly low linear stability, as noted in \cref{sec:stability}.

We maintain three quantities at each step, \(x_n\), \(\hat{x}_n\), and \(k_n\), where \(k_n = f(\hat{x}_n, t_n)\). We initialise \(\hat{x}_0 = x_0\) and \(k_0 = f(x_0, t_0)\). The forward update from \(t_n\) to \(t_{n+1}\) is
\[
\hat{x}_{n+1} = 2x_n - \hat{x}_n + hk_n,
\]
\[
k_{n+1} = f(\hat{x}_{n+1}, t_{n+1}),
\]
\[
x_{n+1} = x_n + \frac{h}{2}\bigl(k_n + k_{n+1}\bigr).
\]
To step backwards from \(t_{n+1}\) to \(t_n\), we use
\[
\hat{x}_n = 2x_{n+1} - \hat{x}_{n+1} - hk_{n+1},
\]
\[
k_n = f(\hat{x}_n, t_n),
\]
\[
x_n = x_{n+1} - \frac{h}{2}\bigl(k_n + k_{n+1}\bigr).
\]
Therefore, in addition to storing \(x_N\) (and the endpoint time \(t_N\)), we also need to store the endpoint values \(\hat{x}_N\) and \(k_N\).

\section{McCallum--Foster}
\label{appendix:mccallum_foster}

The McCallum--Foster scheme \citep{mccallum2024efficient} gives a general
construction for turning any explicit single-step ODE solver into an
algebraically reversible solver. The order of the
resulting method is inherited from the chosen base solver. The method also has
improved linear stability over Reversible Heun, controlled by a coupling parameter \(\zeta\).

Let \(\Psi_h(t,x)\) denote the one-step increment of the chosen base solver, so
that the corresponding non-reversible step is \(x_{n+1}=x_n+\Psi_h(t_n,x_n)\).
We maintain two quantities at each step, \(x_n\) and \(\hat{x}_n\), and initialise
\(\hat{x}_0=x_0\). For \(\zeta\in(0,1]\), the forward update from \(t_n\) to
\(t_{n+1}\) is
\[
x_{n+1}
=
\zeta x_n + (1-\zeta)\hat{x}_n + \Psi_h(t_n,\hat{x}_n),
\]
\[
\hat{x}_{n+1}
=
\hat{x}_n - \Psi_{-h}(t_{n+1},x_{n+1}).
\]
To step backwards from \(t_{n+1}\) to \(t_n\), we use
\[
\hat{x}_n
=
\hat{x}_{n+1} + \Psi_{-h}(t_{n+1},x_{n+1}),
\]
\[
x_n
=
\zeta^{-1}x_{n+1}
+
(1-\zeta^{-1})\hat{x}_n
-
\zeta^{-1}\Psi_h(t_n,\hat{x}_n).
\]
Therefore, in addition to storing \(x_N\) and the endpoint time \(t_N\), we also
need to store the endpoint value \(\hat{x}_N\).

\subsection{Linear stability}
\label{appendix:mccallum_linear_stability}
\citet[Theorem~2.3]{mccallum2024efficient} characterise the linear stability
region of the McCallum--Foster method. Consider the scalar test equation
\(\dot{y}=\lambda y\), with \(\lambda<0\), and let \(z=h\lambda\). Following
their notation, let \(R\) denote the \textit{transfer function} of the base solver
increment, so that
\[
\Psi_h(t_n,y_n)=R(z)y_n .
\]
Then the reversible scheme is linearly stable if and only if
\[
|\Gamma(z)| < 1+\zeta,
\qquad
\Gamma(z)
=
1+\zeta
-
(1-\zeta)R(-z)
-
R(-z)R(z).
\]
Thus, for any explicit Runge--Kutta base method, the stability region is obtained
by substituting the corresponding \(R\). The resulting region is smaller than
that of the base method, but remains non-trivial.

\section{Additional Experiments}
\label{appendix:additional_results}

Beyond the experiments in the main paper, in \cref{appendix:smoothing_methods} we report the full results of our image editing experiments, with and without Smooth Diffusion, and with three additional vector-field smoothing variants. In \cref{appendix:ode_formulation_ablations}, we report the results of a large ablation over the ODE parametrisation, the step discretisation schedule, and whether the solver is applied in a semilinear integrating-factor form or a black-box form. This ablation justifies the solver forms proposed in \cref{sec:ees_solvers_for_diffusion} and highlights alternative forms that may also be used.

We provide additional results for greyscale image editing in \cref{sec:greyscale}, experiments using SDXL in \cref{sec:sdxl}, results across different numbers of function evaluations in \cref{sec:nfe}, and alternative EDICT and BDIA hyperparameter settings in \cref{sec:edict_bdia_hyperparameters}. We also characterise the distribution of the inverted variable in \cref{sec:latent_terminal} and study the effect of varying the guidance scale in \cref{sec:guidance_scale}.

\subsection{Smoothing Method Ablation}
\label{appendix:smoothing_methods}

Here, we study the impact of vector-field smoothing methods previously proposed to improve DDIM inversion and DDIM-based editing, and evaluate them on reversible and near-reversible solvers. To our knowledge, this has not been previously reported. We hope this study helps position reversible solvers within the broader DDIM inversion literature, where smoothing methods are widely used but typically treated as orthogonal to solver design.

It is well known that DDIM inversion error typically increases with the classifier-free guidance scale, and several works propose inference-time tricks to mitigate this effect. Unlike Smooth Diffusion, these approaches do not modify model parameters. Instead, they change how conditioning is applied during inversion and/or sampling.

\paragraph{Vector-field smoothing and guidance tricks.}
\textbf{Negative Prompt Inversion (NPI)} \citep{miyake2025negative} sets \(c_{\texttt{null}} := c_{\texttt{src}}\), so the reverse-direction guidance cancels and the inversion becomes guidance-free:
\(
\hat{\epsilon}_\theta(x_t, t, c_{\texttt{src}})=\epsilon_\theta(x_t, t, c_{\texttt{src}}).
\)
\textbf{Proximal Guidance (Prox)} \citep{han2023improving} further reduces guidance in the forward (editing) direction by setting the guidance scale to \(g=1\) at spatial locations where the conditional predictions for the source and target prompts are already similar, that is, where
\(
\left\lVert \epsilon_\theta(x_t,t,c_{\texttt{trg}})-\epsilon_\theta(x_t,t,c_{\texttt{src}})\right\rVert_2
\)
is small. Intuitively, Prox suppresses unnecessary conditional forcing in regions that are unlikely to change, which can reduce background drift. Both NPI and Prox were originally introduced in the context of DDIM-based editing. Here, we apply them across all solvers. \textbf{Smooth Diffusion} \citep{guo2024smooth} fine-tunes Stable Diffusion to produce smoother trajectories and reports improved DDIM inversion and editing. Since it changes the underlying model dynamics, it is complementary to NPI and Prox.

Throughout this appendix, we use ``smoothing method'' as an umbrella term for both model-based smoothing (Smooth Diffusion) and inference-time guidance tricks (such as NPI and Proximal Guidance) that are designed to reduce inversion error and background drift.

\paragraph{Results and qualitative examples.}
We report PIE-Bench image-editing results for each smoothing method in \cref{tab:solver_metrics_PIEBench_all}. In \cref{fig:metrics_tricks}, we visualise the corresponding metric shifts for each solver under smoothing. We also report CLIP scores for the large prompt-deviation editing task (\cref{sec:large_edits_and_stability}) in \cref{tab:clip_whole_by_setting}.

\begin{table*}[tbp]
\caption{PIE-Bench image editing metrics for the \textit{small edits} task for each solver under different smoothing methods. We report background similarity using PSNR and LPIPS, and edit quality using CLIP similarity on the whole image and on the masked edited region. Within each block, we highlight the \textbf{best} and \underline{second-best} results. See \cref{fig:metrics_tricks} for the corresponding per-solver shifts caused by adding a smoothing method. EES(2,7) generally achieves strongest edit quality scores, while McCallum--Foster (Euler) achieves the strongest background similarity scores, illustrating the trade-off between edit quality and background preservation discussed in \cref{sec:experiments_editing}.}
\label{tab:solver_metrics_PIEBench_all}
\vskip 0.10in
\centering
\scriptsize
\renewcommand{\arraystretch}{1.15}
\setlength{\tabcolsep}{6.0pt}
\resizebox{\textwidth}{!}{%
\begin{tabular}{l|c|c|c|c|c}
\hline\hline
\multicolumn{1}{c|}{\textbf{Solver}} &
\begin{tabular}[c]{@{}c@{}}\textbf{Stable Diffusion 1.5}\end{tabular} &
\begin{tabular}[c]{@{}c@{}}\textbf{Stable Diffusion 1.5}\\ \textbf{+ Smooth Diffusion}\end{tabular} &
\begin{tabular}[c]{@{}c@{}}\textbf{Stable Diffusion 1.5}\\ \textbf{+ Negative Prompt Inversion}\end{tabular} &
\begin{tabular}[c]{@{}c@{}}\textbf{Stable Diffusion 1.5}\\ \textbf{+ Proximal Guidance}\end{tabular} &
\begin{tabular}[c]{@{}c@{}}\textbf{Stable Diffusion 1.5}\\ \textbf{+ Smooth Diffusion}\\ \textbf{+ Proximal Guidance}\end{tabular} \\
\hline\hline

\noalign{\vskip 0.35ex}
\multicolumn{6}{c}{\textbf{Edit Quality}} \\
\noalign{\vskip 0.35ex}
\hline\hline

\multicolumn{6}{c}{\textbf{CLIP Similarity: Whole}$\,\uparrow$} \\
\hline
EDICT                         & 24.15 & 24.17 & \underline{24.51} & 24.48 & 24.19 \\
Reversible Heun               & 24.27 & 23.89 & 24.28 & 24.19 & 23.86 \\
BDIA                          & \underline{24.33} & 24.07 & 24.34 & 24.25 & 24.02 \\
BELM                          & 23.91 & 24.21 & 24.32 & 24.28 & 24.25 \\
DDIM                          & 23.85 & \underline{24.53} & 24.49 & 24.42 & \underline{24.38} \\
Rex (Euler)                     & 22.67 & 23.57 & 23.93 & 23.99 & 23.72 \\
Rex (Midpoint)                  & 22.93 & \textbf{24.54} & 24.50 & \underline{24.51} & \textbf{24.52} \\
McCallum--Foster (Euler)        & 23.24 & 23.18 & 23.20 & 23.14 & 23.08 \\
McCallum--Foster (Midpoint)     & 23.49 & 23.59 & 23.68 & 23.65 & 23.63 \\
EES(2,5)                      & 24.21 & 24.37 & 24.47 & 24.38 & 24.25 \\
EES(2,7)                      & \textbf{24.36} & 24.47 & \textbf{24.60} & \textbf{24.52} & 24.35 \\
\hline

\multicolumn{6}{c}{\textbf{CLIP Similarity: Edited}$\,\uparrow$} \\
\hline
EDICT                         & 21.45 & 21.26 & 21.73 & 21.66 & 21.26 \\
Reversible Heun               & 21.41 & 21.05 & 21.43 & 21.33 & 20.99 \\
BDIA                          & 21.41 & 21.11 & 21.38 & 21.33 & 21.08 \\
BELM                          & 21.02 & 21.28 & 21.56 & 21.51 & 21.41 \\
DDIM                          & 21.33 & 21.59 & 21.78 & 21.71 & 21.41 \\
Rex (Euler)                     & 20.16 & 20.64 & 20.78 & 20.85 & 20.87 \\
Rex (Midpoint)                  & 20.56 & \underline{21.62} & \underline{21.86} & \underline{21.72} & \textbf{21.62} \\
McCallum--Foster (Euler)        & 20.50 & 20.42 & 20.44 & 20.39 & 20.31 \\
McCallum--Foster (Midpoint)     & 20.70 & 20.68 & 20.85 & 20.85 & 20.80 \\
EES(2,5)                      & \underline{21.57} & 21.54 & 21.77 & 21.65 & 21.44 \\
EES(2,7)                      & \textbf{21.80} & \textbf{21.64} & \textbf{21.89} & \textbf{21.83} & \underline{21.51} \\
\hline\hline

\noalign{\vskip 0.8ex}
\multicolumn{6}{c}{\textbf{Background Similarity}} \\
\noalign{\vskip 0.35ex}
\hline\hline

\multicolumn{6}{c}{\textbf{PSNR}$\,\uparrow$} \\
\hline
EDICT                         & 25.18 & 28.12 & 26.18 & 26.54 & 28.45 \\
Reversible Heun               & 27.59 & 28.89 & 28.22 & 28.52 & \underline{29.29} \\
BDIA                          & \underline{27.99} & \underline{28.91} & \underline{28.33} & \underline{28.59} & 29.25 \\
BELM                          & 22.55 & 25.90 & 24.17 & 24.55 & 27.01 \\
DDIM                          & 22.36 & 26.20 & 24.27 & 24.61 & 27.06 \\
Rex (Euler)                     & 21.74 & 26.37 & 23.97 & 23.78 & 26.69 \\
Rex (Midpoint)                  & 21.79 & 26.35 & 23.90 & 24.02 & 27.17 \\
McCallum--Foster (Euler)        & \textbf{29.54} & \textbf{29.80} & \textbf{29.75} & \textbf{29.87} & \textbf{30.02} \\
McCallum--Foster (Midpoint)     & 27.74 & 28.81 & 28.32 & 28.45 & 29.00 \\
EES(2,5)                      & 23.38 & 27.26 & 25.25 & 25.58 & 27.97 \\
EES(2,7)                      & 22.75 & 26.89 & 24.64 & 24.98 & 27.70 \\
\hline

\multicolumn{6}{c}{\textbf{LPIPS}$\times 10^3\,\downarrow$} \\
\hline
EDICT                         & 79.91 & 45.21 & 66.55 & 61.71 & 43.05 \\
Reversible Heun               & 52.45 & \underline{39.50} & 46.44 & \underline{43.63} & \underline{37.11} \\
BDIA                          & \underline{50.52} & 40.86 & 47.76 & 44.82 & 38.58 \\
BELM                          & 131.14 & 75.24 & 104.48 & 97.65 & 61.93 \\
DDIM                          & 148.43 & 81.43 & 116.97 & 109.76 & 70.60 \\
Rex (Euler)                     & 136.31 & 60.31 & 96.05 & 110.52 & 65.41 \\
Rex (Midpoint)                  & 143.03 & 67.88 & 105.77 & 105.92 & 59.80 \\
McCallum--Foster (Euler)        & \textbf{36.80} & \textbf{34.50} & \textbf{35.55} & \textbf{34.85} & \textbf{33.62} \\
McCallum--Foster (Midpoint)     & 52.33 & 41.37 & \underline{46.38} & 46.51 & 41.46 \\
EES(2,5)                      & 117.03 & 59.30 & 87.63 & 81.32 & 50.83 \\
EES(2,7)                      & 126.89 & 62.52 & 96.16 & 89.09 & 53.10 \\
\hline\hline
\end{tabular}%
}
\vskip -0.10in
\end{table*}

\begin{figure}[tbp]
  \centering

  \begin{subfigure}[b]{0.49\linewidth}
    \centering
    \includegraphics[width=\linewidth]{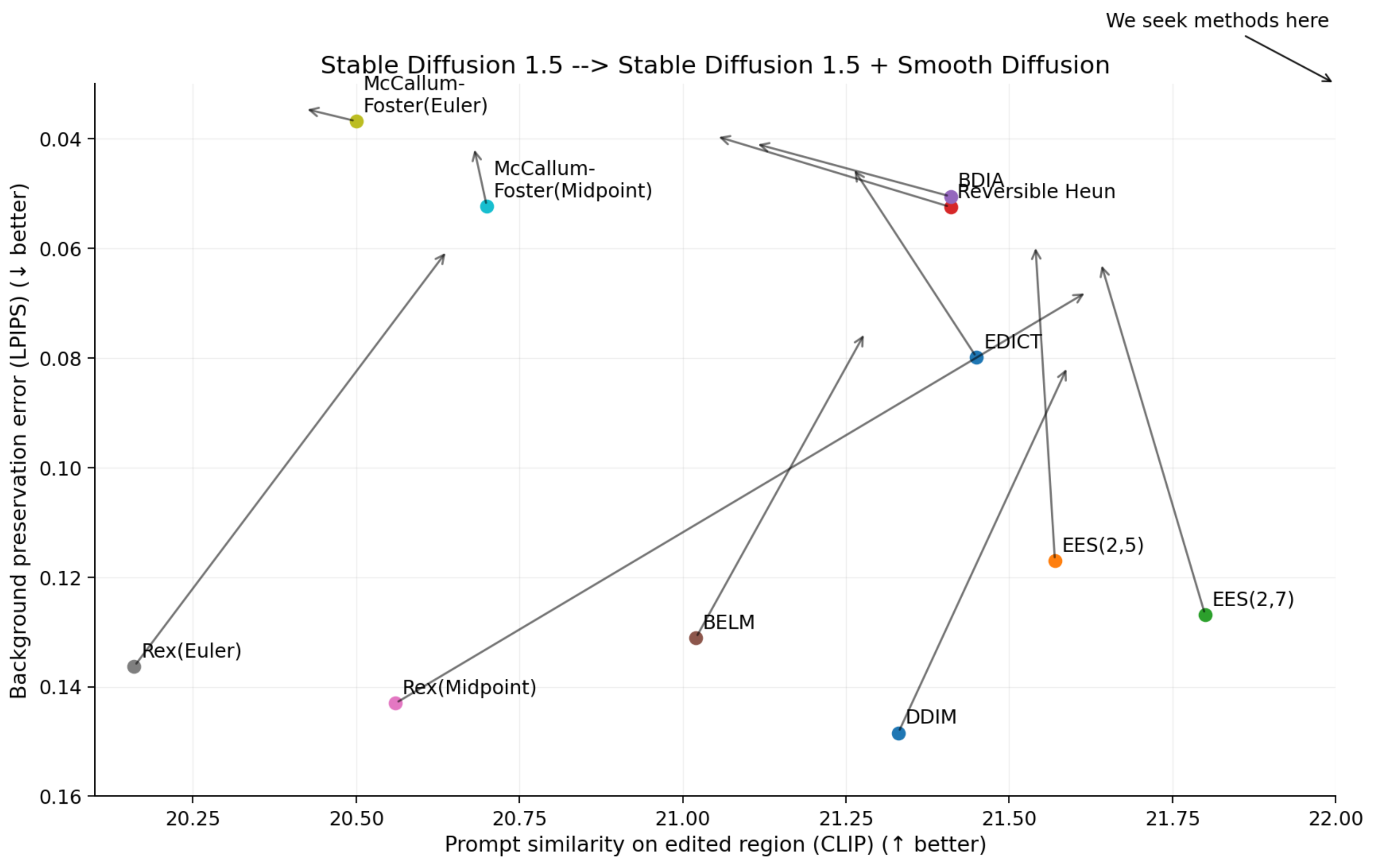}
  \end{subfigure}\hfill
  \begin{subfigure}[b]{0.49\linewidth}
    \centering
    \includegraphics[width=\linewidth]{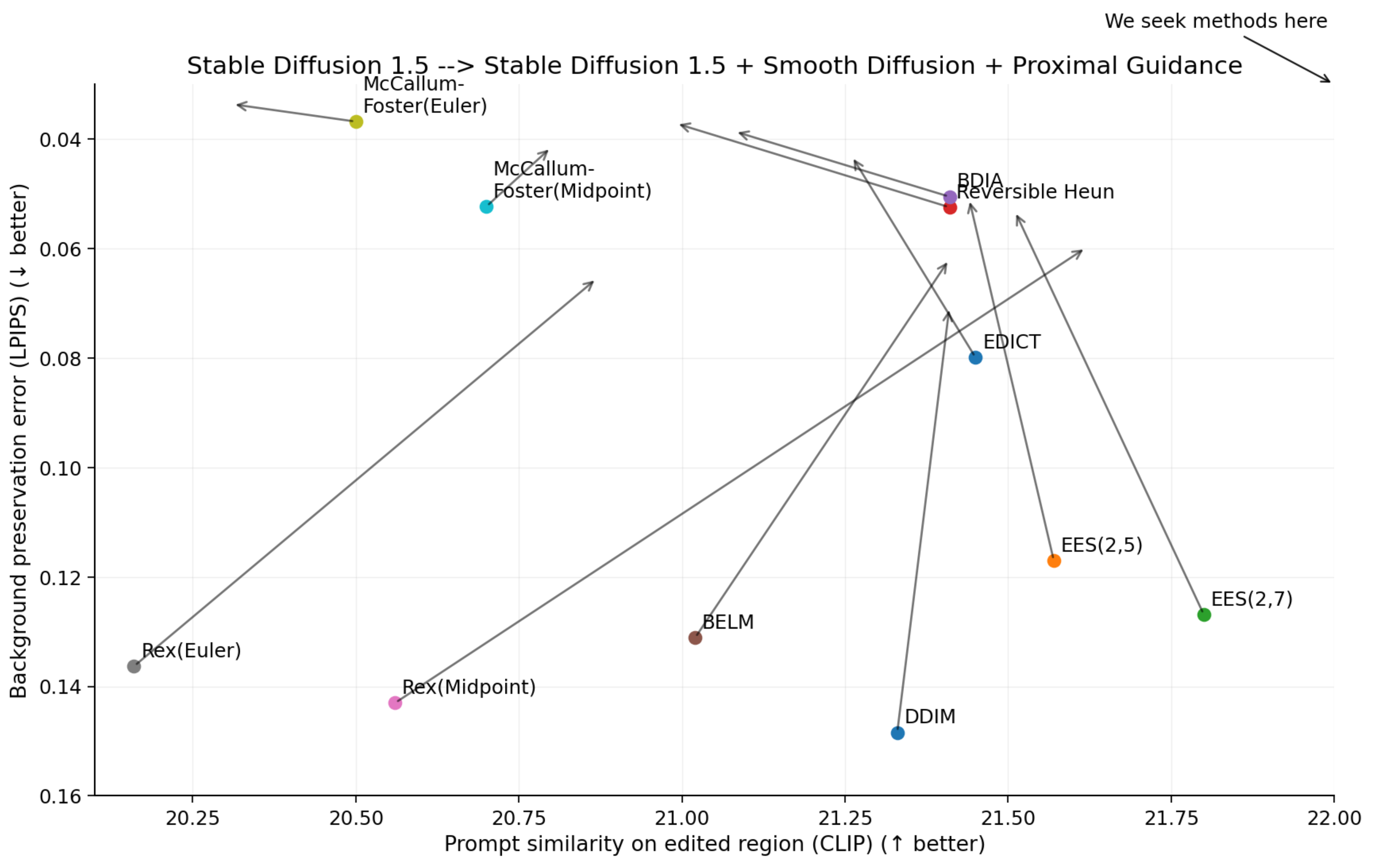}
  \end{subfigure}

  \medskip

  \begin{subfigure}[b]{0.49\linewidth}
    \centering
    \includegraphics[width=\linewidth]{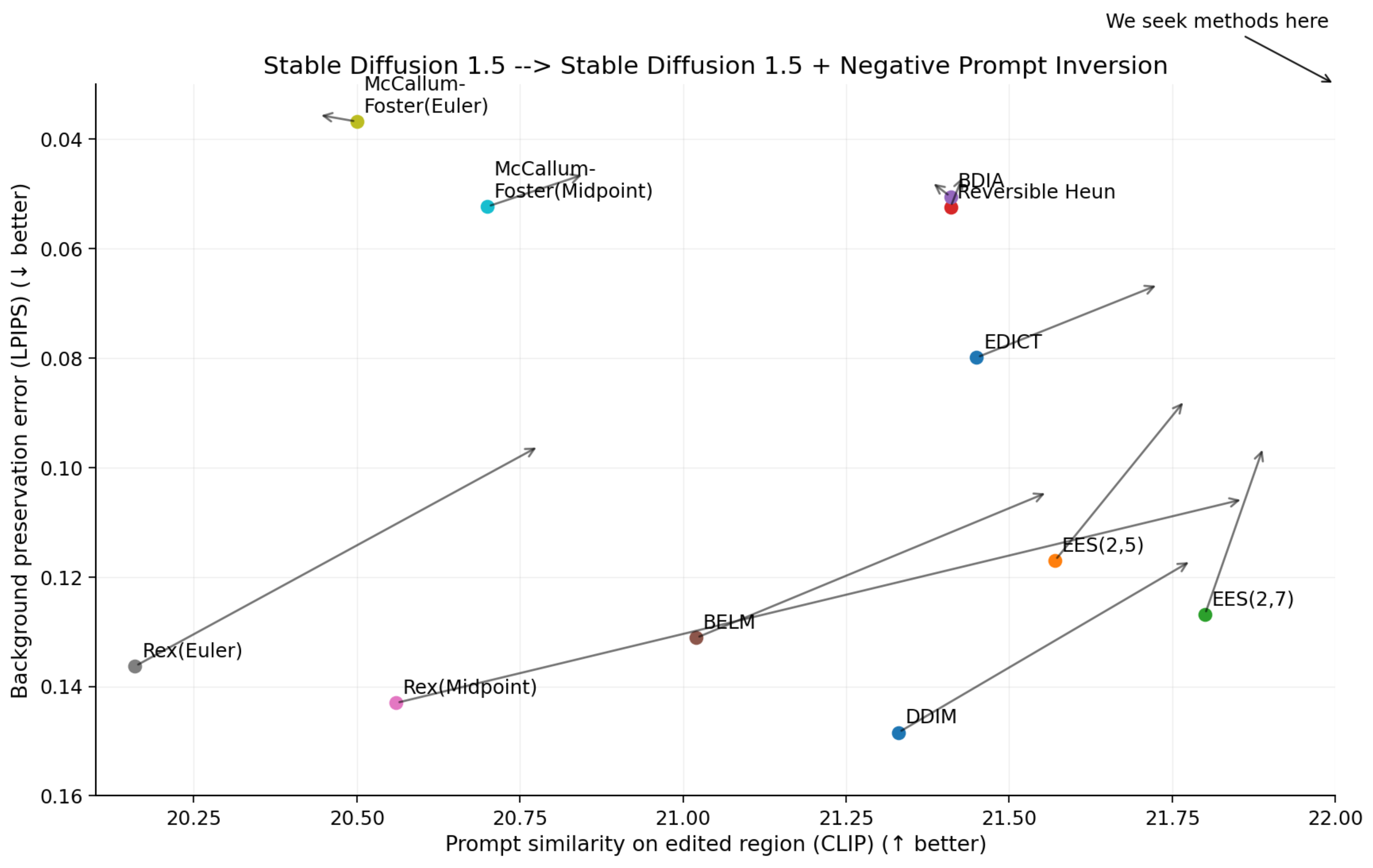}
  \end{subfigure}\hfill
  \begin{subfigure}[b]{0.49\linewidth}
    \centering
    \includegraphics[width=\linewidth]{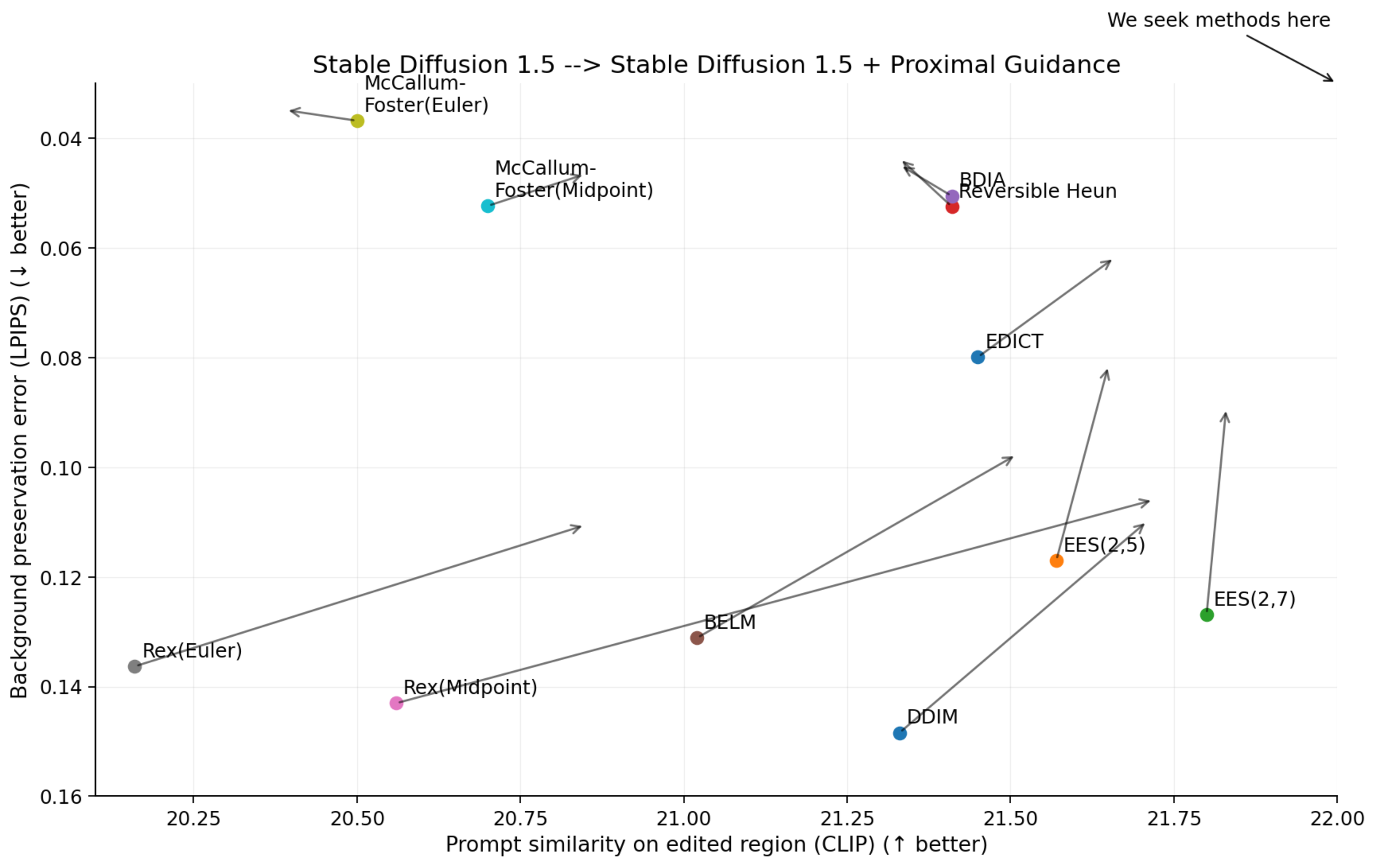}
  \end{subfigure}

  \caption{Effect of each smoothing method on PIE-Bench metrics for each solver (arrows indicate the shift from no smoothing to the corresponding smoothing method). From top-left to bottom-right: Smooth Diffusion, Smooth Diffusion + Proximal Guidance, NPI, and Proximal Guidance. We are seeking methods in the top right corner.}
  \label{fig:metrics_tricks}
\end{figure}

\begin{table}[tbp]
\caption{Image editing metrics for the \textit{large edits} task (\cref{sec:large_edits_and_stability}) under different smoothing methods. Since the original edit masks are no longer meaningful, we use CLIP similarity on the full image as our main metric. We report DDIM for reference, but the comparison is among reversible and near-reversible solvers, where instability under large edits is most common. We highlight the \textbf{best} and \underline{second-best} (near) reversible solvers.}
\label{tab:clip_whole_by_setting}
\vskip 0.10in
\centering
\scriptsize
\renewcommand{\arraystretch}{1.15}
\setlength{\tabcolsep}{3.0pt}
\resizebox{\linewidth}{!}{%
\begin{tabular}{l|c|c|c|c|c}
\hline\hline
\textbf{Solver} &
\multicolumn{5}{c}{\textbf{CLIP Similarity (Whole)} $\uparrow$} \\
\cline{2-6}
&
\begin{tabular}[c]{@{}c@{}}\textbf{Stable Diffusion 1.5}\\ \end{tabular} &
\begin{tabular}[c]{@{}c@{}}\textbf{Stable Diffusion 1.5}\\ \textbf{+ Smooth Diffusion}\end{tabular} &
\begin{tabular}[c]{@{}c@{}}\textbf{Stable Diffusion 1.5}\\ \textbf{+ Negative Prompt Inversion}\end{tabular} &
\begin{tabular}[c]{@{}c@{}}\textbf{Stable Diffusion 1.5}\\ \textbf{+ Proximal Guidance}\end{tabular} &
\begin{tabular}[c]{@{}c@{}}\textbf{Stable Diffusion 1.5}\\ \textbf{+ Smooth Diffusion}\\ \textbf{+ Proximal Guidance}\end{tabular} \\
\hline

EDICT                    & 23.11 & 21.88 & 23.85 & 23.57 & 22.11\\
Reversible Heun          & 21.74 & 20.12 & 21.58 & 21.03 & 19.51\\
BDIA                     & 21.09 & 20.77 & 21.11 & 20.55 & 20.39\\
BELM                     & 20.18 & 22.19 & 22.98 & 22.71 & 22.28\\
Rex (Euler)                & 15.39 & 15.17 & 15.29 & 15.54 & 15.22\\
Rex (Midpoint)             & 18.45 & 22.77 & 22.42 & 20.22 & 20.95\\
McCallum--Foster (Euler)   & 12.44 & 11.49 & 11.85 & 11.42 & 10.95\\
McCallum--Foster (Midpoint)& 16.02 & 16.05 & 15.02 & 14.65 & 15.12\\
EES(2,5)                 & \underline{24.12} & \underline{23.63} & \textbf{25.15} & \underline{24.77} & \underline{23.31}\\
EES(2,7)                 & \textbf{24.40} & \textbf{23.66} & \underline{25.05} & \textbf{24.83} & \textbf{23.39}\\
\hline

\multicolumn{6}{c}{\rule{0pt}{3.8mm}}\\[-1.2mm]
\hline
DDIM                     & 23.97 & 23.90 & 25.30 & 25.12 & 23.78\\

\hline\hline
\end{tabular}%
}
\vskip -0.10in
\end{table}

\subsection{ODE Formulation Ablations}
\label{appendix:ode_formulation_ablations}
In this section, we summarise the design choices underlying our proposed diffusion samplers—namely the ODE parametrisation (\cref{appendix:ode_parametrisation}), the step discretisation schedule (\cref{appendix:discretisation_schedule}), and whether the solver is applied in a semilinear integrating-factor form or a black-box form (\cref{appendix:semilinear_vs_black_box}). We then report, in \cref{sec:ablation_results}, a systematic evaluation of these variants for each proposed solver.

\subsubsection{ODE Parametrisation}
\label{appendix:ode_parametrisation}

For an unknown data distribution \(q_0(x_0)\) of $d$-dimensional images \(x_0 \in \mathbb{R}^d\), we consider a variance-preserving (VP) diffusion model whose forward perturbation marginals at time \(t \in [0, T]\), \(T>0\) satisfy
\[
q_{t}(x_t\mid x_0)=\mathcal N(x_t\mid \alpha_t x_0,\sigma_t^2 I),
\qquad \alpha_t^2+\sigma_t^2=1,
\qquad \alpha_t, \sigma_t>0
\]
so that \(x_t=\alpha_t x_0+\sigma_t\epsilon\) with \(\epsilon\sim\mathcal N(0,I)\).
The corresponding probability-flow ODE can be written in the standard form
\[
\frac{dx_t}{dt}=f(t)\,x_t+\frac{g(t)^2}{2\sigma_t}\,\epsilon_\theta(x_t,t),
\qquad
f(t)=\frac{d\log \alpha_t}{dt},
\quad
g(t)^2=\frac{d\sigma_t^2}{dt}-2f(t)\sigma_t^2.
\]

From \(\alpha_t^2+\sigma_t^2=1\), we then have
\[
\frac{d\sigma_t^2}{dt}=-\frac{d\alpha_t^2}{dt}=-2\alpha_t\frac{d\alpha_t}{dt}=-2\alpha_t^2 f(t),
\quad\Longrightarrow\quad
g(t)^2=-2f(t),
\]
and therefore
\begin{equation}
\frac{dx_t}{dt}
=
f(t)\Bigl(x_t-\epsilon_\theta(x_t,t)/\sigma_t\Bigr)
=
\frac{d}{dt}\bigl[\log \alpha_t\bigr]\,
\Bigl(x_t-\epsilon_\theta(x_t,t)/\sigma_t\Bigr).
\label{eq:ode_t_original}
\end{equation}

\paragraph{Discrete-time schedules.}
The above ODE form does not prevent us from supporting discrete-time diffusion models, for which the noise schedule is defined at a discrete set of times \(\alpha_0, \alpha_1\ldots,\alpha_T\). Following the implementation in DPM \citep{lu2022dpm}, we construct a continuous \(\alpha_t\) by linearly interpolating
\(\log\alpha_t\) between neighbouring discrete timesteps.
We then set \(\sigma_t=\sqrt{1-\alpha_t^2}\).
While \eqref{eq:ode_t_original} is convenient conceptually, it requires approximating the gradient \(d(\log\alpha_t)/dt\), as part of the vector field. We therefore consider the following equivalent reparametrisations.

\paragraph{Half-logSNR parametrisation (\(\epsilon\)-prediction).}
Define the variable \(\lambda_t\) as
\[
\lambda_t:=\log\frac{\alpha_t}{\sigma_t}=\log\alpha_t-\log\sigma_t,
\]
which is strictly decreasing in \(t\) and hence invertible.
For VP schedules one can show \(\frac{d\log \alpha_\lambda}{d\lambda}=\sigma_\lambda^2\), and by change-of-variables the ODE becomes
\begin{equation}
\frac{dx_\lambda}{d\lambda}
=
\sigma_\lambda^2\,x_\lambda-\sigma_\lambda\,\epsilon_\theta(x_\lambda,\lambda).
\label{eq:ode_lambda_eps}
\end{equation}
This is the \(\lambda\)-domain diffusion ODE used by DPM-Solver \citep{lu2022dpm}.

\paragraph{Half-logSNR parametrisation (\(x_0\)-prediction).}
DPM-Solver++ \citep{lu2025dpm} instead solves the equivalent ODE expressed in terms of the data predictor
\[
x_\theta(x_t,t):=\frac{x_t-\sigma_t\epsilon_\theta(x_t,t)}{\alpha_t}.
\]
Substituting \(\sigma_\lambda\epsilon_\theta=x_\lambda-\alpha_\lambda x_\theta\) into \eqref{eq:ode_lambda_eps} gives
\begin{equation}
\frac{dx_\lambda}{d\lambda}
=
-\alpha_\lambda^2\,x_\lambda+\alpha_\lambda\,x_\theta(x_\lambda,\lambda).
\label{eq:ode_lambda_x0}
\end{equation}

\paragraph{Scaled-noise parametrisation.}
One can also define the integration variable
\[
u_t:=\frac{\sigma_t}{\alpha_t}=e^{-\lambda_t},
\qquad
z_t:=\frac{x_t}{\alpha_t}.
\]
Then \(z_t=x_0+u_t\epsilon\), and using \(x_\theta=z_t-u_t\epsilon_\theta\) we obtain the DDIM-style ODE used in BELM \citep{wang2024belm}
\begin{equation}
\frac{dz}{du}
=
\epsilon_\theta(x_t,t)
=
\frac{z-x_\theta(z,u)}{u},
\qquad x_t=\alpha_t z.
\label{eq:ode_u_ddim}
\end{equation}
Notably, this parametrisation removes the linear drift term entirely, so it can be treated as a purely non-linear ODE.

\subsubsection{Semilinear vs.\ black-box}
\label{appendix:semilinear_vs_black_box}
A semilinear ODE in a variable \(\lambda\) has the form
\[
\frac{dx_\lambda}{d\lambda}=a_\lambda x_\lambda+b(x_\lambda,\lambda),
\]
where \(b\) is non-linear (and in our setting contains a neural network).
Black-box solvers (e.g.\ standard Runge--Kutta) discretise the full right-hand side, introducing
discretisation error in both the linear and non-linear terms.
In contrast, one can also solve such ODE using the integrating factor method, which solves the linear part exactly and leaves only a weighted integral of the non-linear part.

\paragraph{The integrating factor method.}
Define
\[
I_\lambda=\exp\!\Bigl(-\int^{\lambda} a_u\,du\Bigr),
\qquad
y_\lambda=I_\lambda x_\lambda.
\]
Then \(y_\lambda\) satisfies
\[
\frac{dy_\lambda}{d\lambda}
=
I_\lambda\,b\!\Bigl(\frac{y_\lambda}{I_\lambda},\lambda\Bigr),
\qquad
x_\lambda=\frac{y_\lambda}{I_\lambda}.
\]

Since equations \cref{eq:ode_lambda_eps} and \cref{eq:ode_lambda_x0} are semilinear and the integrating factors simplify to \(\frac{1}{\alpha_\lambda}\) and \(\frac{1}{\sigma_\lambda}\), this method is of great convenience. We now present the solutions of the respective equations using the integrating factor method.

\paragraph{The solution to \(\epsilon\)-prediction in \(\lambda\).}
For \eqref{eq:ode_lambda_eps}, we have \(a_\lambda=\sigma_\lambda^2\) and \(b(x,\lambda)=-\sigma_\lambda\epsilon_\theta(x,\lambda)\).
Using \(d\log\alpha_\lambda/d\lambda=\sigma_\lambda^2\), the integrating factor is \(I_\lambda =  1/\alpha_\lambda\), so taking
\[
y_\lambda=\frac{x_\lambda}{\alpha_\lambda}
\]
gives
\begin{equation}
\frac{dy_\lambda}{d\lambda}
=
-\frac{\sigma_\lambda}{\alpha_\lambda}\,\epsilon_\theta(x_\lambda,\lambda)
=
-\,e^{-\lambda}\,\epsilon_\theta(x_\lambda,\lambda),
\qquad x_\lambda=\alpha_\lambda y_\lambda.
\label{eq:ode_lambda_eps_semilinear}
\end{equation}
Integrating over \([\lambda_s,\lambda_t]\) yields
\[
y_{\lambda_t}=y_{\lambda_s}-\int_{\lambda_s}^{\lambda_t} e^{-u}\,\epsilon_\theta(x_u,u)\,du,
\]
and mapping back to \(x_\lambda\) recovers the DPM-Solver exact-solution form
\begin{equation}
x_{\lambda_t}
=
\frac{\alpha_{\lambda_t}}{\alpha_{\lambda_s}}\,x_{\lambda_s}
-\alpha_{\lambda_t}\int_{\lambda_s}^{\lambda_t} e^{-u}\,\epsilon_\theta(x_u,u)\,du.
\label{eq:ode_lambda_eps_exact}
\end{equation}

\paragraph{The solution to \(x_0\)-prediction in \(\lambda\).}
For \eqref{eq:ode_lambda_x0}, we have \(a_\lambda=-\alpha_\lambda^2\) and \(b(x,\lambda)=\alpha_\lambda x_\theta(x,\lambda)\).
Since \(d\log\sigma_\lambda/d\lambda=-\alpha_\lambda^2\), we get \(I_\lambda= 1/\sigma_\lambda\), and with
\[
y_\lambda=\frac{x_\lambda}{\sigma_\lambda}
\]
we obtain
\begin{equation}
\frac{dy_\lambda}{d\lambda}
=
\frac{\alpha_\lambda}{\sigma_\lambda}\,x_\theta(x_\lambda,\lambda)
=
e^{\lambda}\,x_\theta(x_\lambda,\lambda).
\label{eq:ode_lambda_x0_semilinear}
\end{equation}
Hence, the corresponding exact-solution form is
\begin{equation}
x_{\lambda_t}
=
\frac{\sigma_{\lambda_t}}{\sigma_{\lambda_s}}\,x_{\lambda_s}
+\sigma_{\lambda_t}\int_{\lambda_s}^{\lambda_t} e^{u}\,x_\theta(x_u,u)\,du,
\label{eq:ode_lambda_x0_exact}
\end{equation}
which is precisely the DPM-Solver++ reformulation.

\paragraph{``Semilinear'' solvers.}
When we report the metrics for a \textit{semilinear} solver in \cref{sec:ablation_results}, we apply the numerical method to the transformed ODE in \(y\)
(i.e. \eqref{eq:ode_lambda_eps_semilinear}, \eqref{eq:ode_lambda_x0_semilinear}) and then map back to \(x\).
When we report a black-box solver, we apply the same numerical method directly to the original ODE in \(x\)
(i.e.\ \eqref{eq:ode_lambda_eps}, \eqref{eq:ode_lambda_x0}).
Finally, the scaled-noise parametrisation \eqref{eq:ode_u_ddim} has no linear drift term, so it does not require an integrating factor.

\subsubsection{Step discretisation schedule}
\label{appendix:discretisation_schedule}
For each ODE parametrisation, we consider two choices of step schedule.

\paragraph{Uniform steps in the integration variable.}
Given an integration variable \(\tau\) (such as \(\lambda, u\)), we take \(\tau\)-uniform steps by linearly spacing the grid
\(\{\tau_0,\dots,\tau_N\}\) between the endpoints and applying the solver with constant step size
\(\Delta\tau=(\tau_N-\tau_0)/N\).

\paragraph{Steps induced by a uniform \(t\)-grid.}
Alternatively, we start from a uniform grid in the original diffusion time \(t\), i.e.\ \(t_i=i/N\), and map it to the
integration variable \(\tau\) via \(\tau_i:=\tau(t_i)\).
This yields a non-uniform step size \(\Delta\tau_i=\tau_{i+1}-\tau_i\).

\subsubsection{Results}
\label{sec:ablation_results}
To evaluate a given ODE configuration, we run the 140 editing tasks in the \textit{random images} category of PIE-Bench and report CLIP and LPIPS. \cref{fig:ablations} shows a systematic evaluation over the three design dimensions introduced above and motivates the configurations used in \cref{sec:ees_solvers_for_diffusion}. For the EES schemes, edit fidelity is already strong across settings, so we select the variant that most consistently improves background preservation. This corresponds to the semilinear formulation with $\lambda$-uniform discretisation and the half-logSNR $x$-parameterisation for both EES(2,5) and EES(2,7). Across solvers, $\lambda$-uniform discretisation tends to yield better reversibility than $t$-uniform discretisation, while the remaining design dimensions have a smaller effect.

For Reversible Heun, which is already strong in reversibility relative to other solvers, we choose the configuration with the best overall balance of edit alignment and background preservation. This is the black-box formulation with $t$-uniform discretisation and the half-logSNR $x$-parameterisation. The difference between half-logSNR $x$- and $\epsilon$-parameterisations is minor in this setting. We also observe that, for Reversible Heun, the black-box formulation generally outperforms the semilinear one. We use the same choices for our implementation of McCallum-Foster (Euler) and McCallum-Foster (Midpoint). 

We have observed slow convergence for the semilinear formulation when using the half-logSNR in \(x\) parametrisation with a $t$-uniform discretisation, and the scaled-noise parametrisation with a $u$-uniform discretisation. They do not produce satisfactory images in our setting of $48$ model evaluations. We therefore omit these combinations from further study.

\begin{figure}[tbp]
  \centering
  \begin{subfigure}[tbp]{0.6\linewidth}
    \centering
    \includegraphics[width=\linewidth]{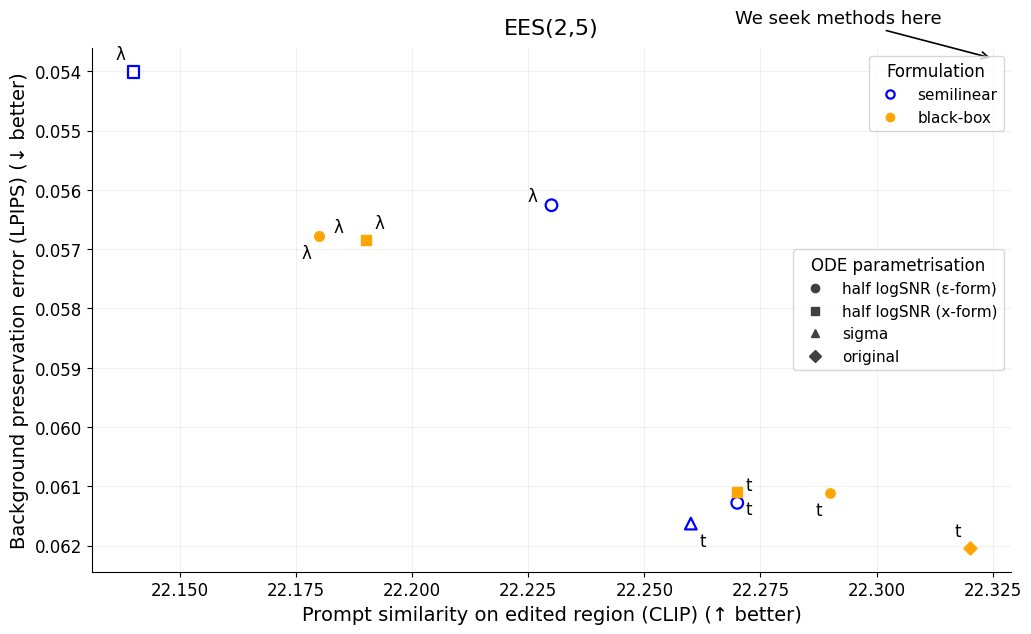}
  \end{subfigure}\hfill
  \begin{subfigure}[tbp]{0.6\linewidth}
    \centering
    \includegraphics[width=\linewidth]{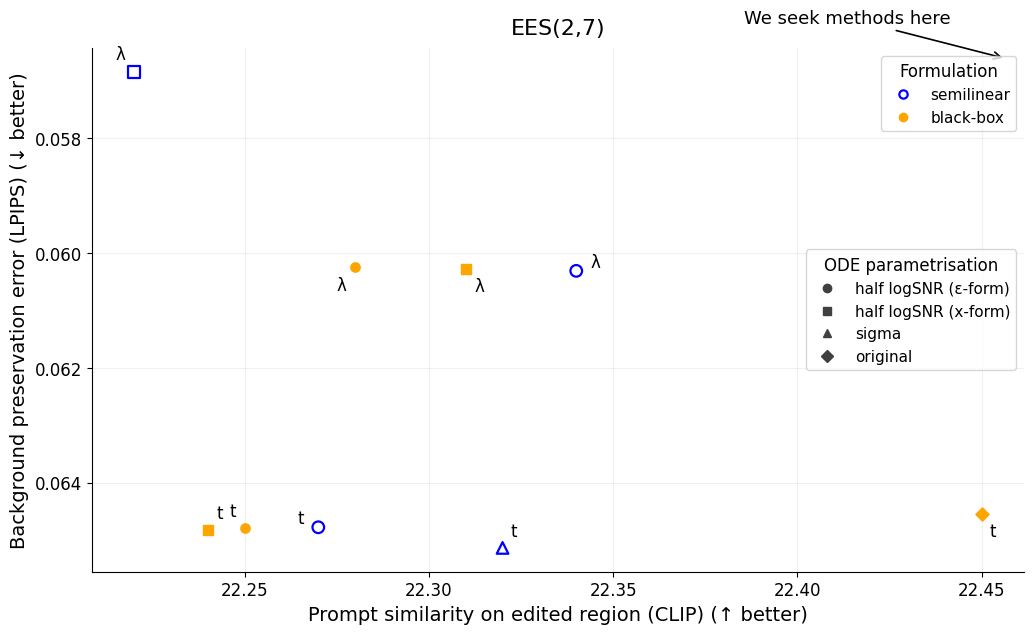}
  \end{subfigure}

  \vspace{0.8ex}

  \begin{subfigure}[tbp]{0.6\linewidth}
    \centering
    \includegraphics[width=\linewidth]{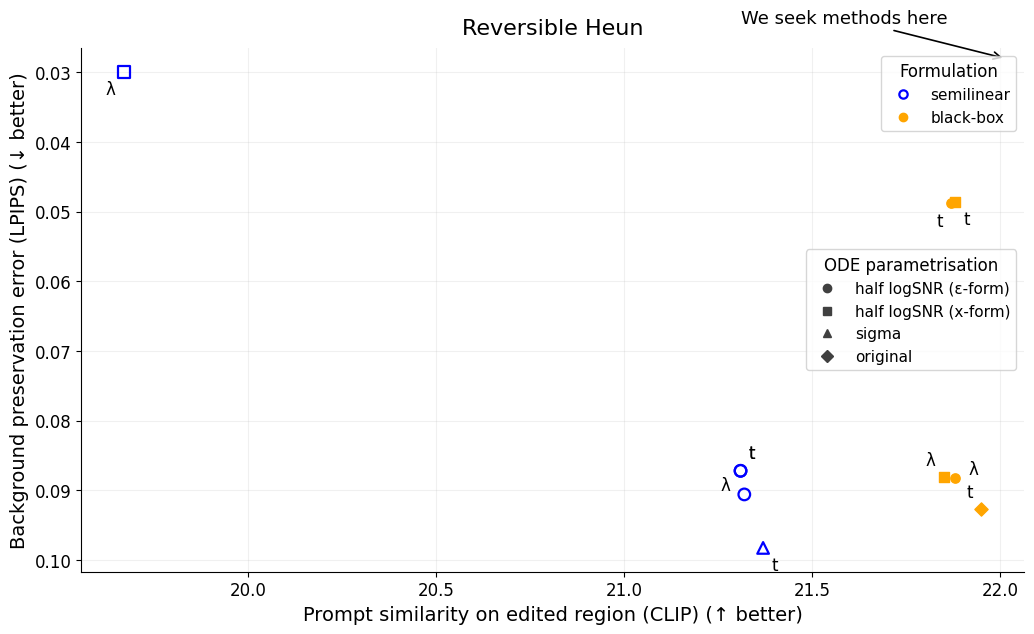}
  \end{subfigure}

  \caption{Ablations over solver configuration on the \textit{random images} category of PIE-Bench. Each panel plots edit prompt alignment against background preservation for EES(2,5), EES(2,7), and Reversible Heun. We vary three separate dimensions: ODE parametrisation (marker shape: original \eqref{eq:ode_t_original}, half logSNR in $\epsilon$ \eqref{eq:ode_lambda_eps} and $x$ \eqref{eq:ode_lambda_x0}, sigma \eqref{eq:ode_u_ddim}), discretisation variable (text label $t/\lambda$ indicates uniform steps in that variable), and formulation (colour indicates semilinear vs black-box). We report this using Smooth Diffusion for EES and Stable Diffusion 1.5 for Reversible Heun, since \cref{appendix:smoothing_methods} shows that smoothing methods do not improve Reversible Heun.}
  \label{fig:ablations}
\end{figure}

\subsection{Greyscale}
\label{sec:greyscale}

We additionally observe failures of non-EES reversible solvers even when the edit prompt is semantically close to the source prompt but requires a large structured visual change in the image (e.g., removing colour). In the greyscale conversion task, these solvers often retain residual colour or introduce artefacts, whereas EES produces a more consistent greyscale edit. \Cref{fig:example_figure} shows a representative example, and an extended grid with all solvers and additional examples is provided in \cref{fig:black_and_white_extended}.

\begin{figure}[tbp]
    \centering
    \includegraphics[width=\linewidth]{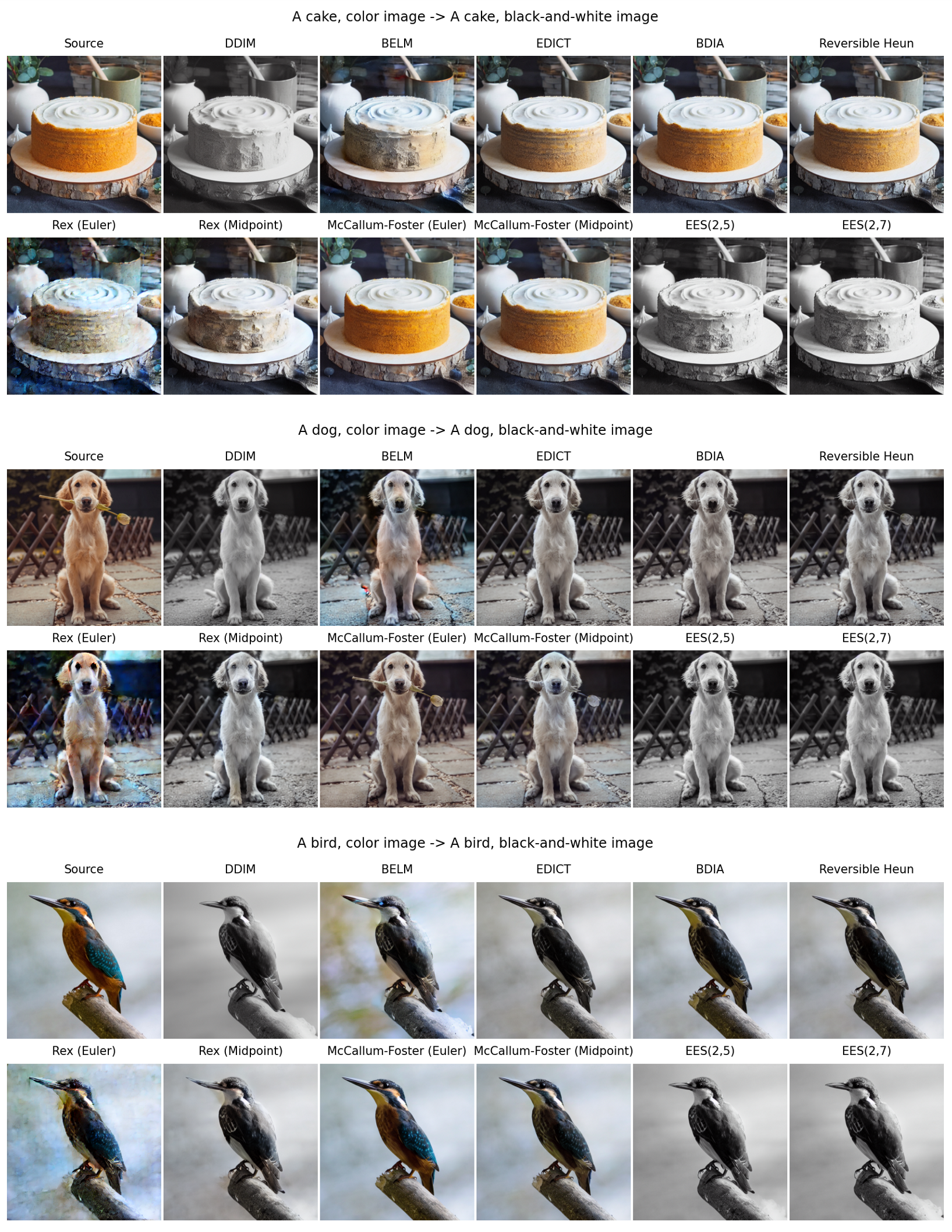}
    \caption{Extended grid for examples in \cref{fig:example_figure}. Greyscale conversion examples under Smooth Diffusion.}
    \label{fig:black_and_white_extended}
\end{figure}

\subsection{Stable Diffusion XL}
\label{sec:sdxl}
To show consistency across models, we reproduce the main results in \cref{fig:metrics_trade_off}, \cref{tab:solver_metrics}, and \cref{tab:clip_whole_sd_vs_smooth} using the Stable Diffusion XL (SDXL) \citep{podell2023sdxl} model instead of Stable Diffusion 1.5. In \cref{fig:sdxl}, we show the background preservation and edit quality trade-off on the \textit{small edits} task. Since there is no Smooth Diffusion checkpoint for SDXL, we instead use NPI as our smoothing method of choice, which we have shown in \cref{appendix:smoothing_methods} to have a similar effect. As in \cref{sec:experiments_editing}, we use the smoothing method only for the solvers for which it is beneficial. In \cref{tab:main_solver_metrics_sdxl_npi}, we show the equivalent of \cref{tab:solver_metrics} and \cref{tab:clip_whole_sd_vs_smooth} with all metrics for SDXL. For large edits, we use NPI with all solvers. We run all the experiments in this section on Category 0 of PIE-Bench.

\begin{figure}[tbp]
    \centering
    \includegraphics[width=\linewidth]{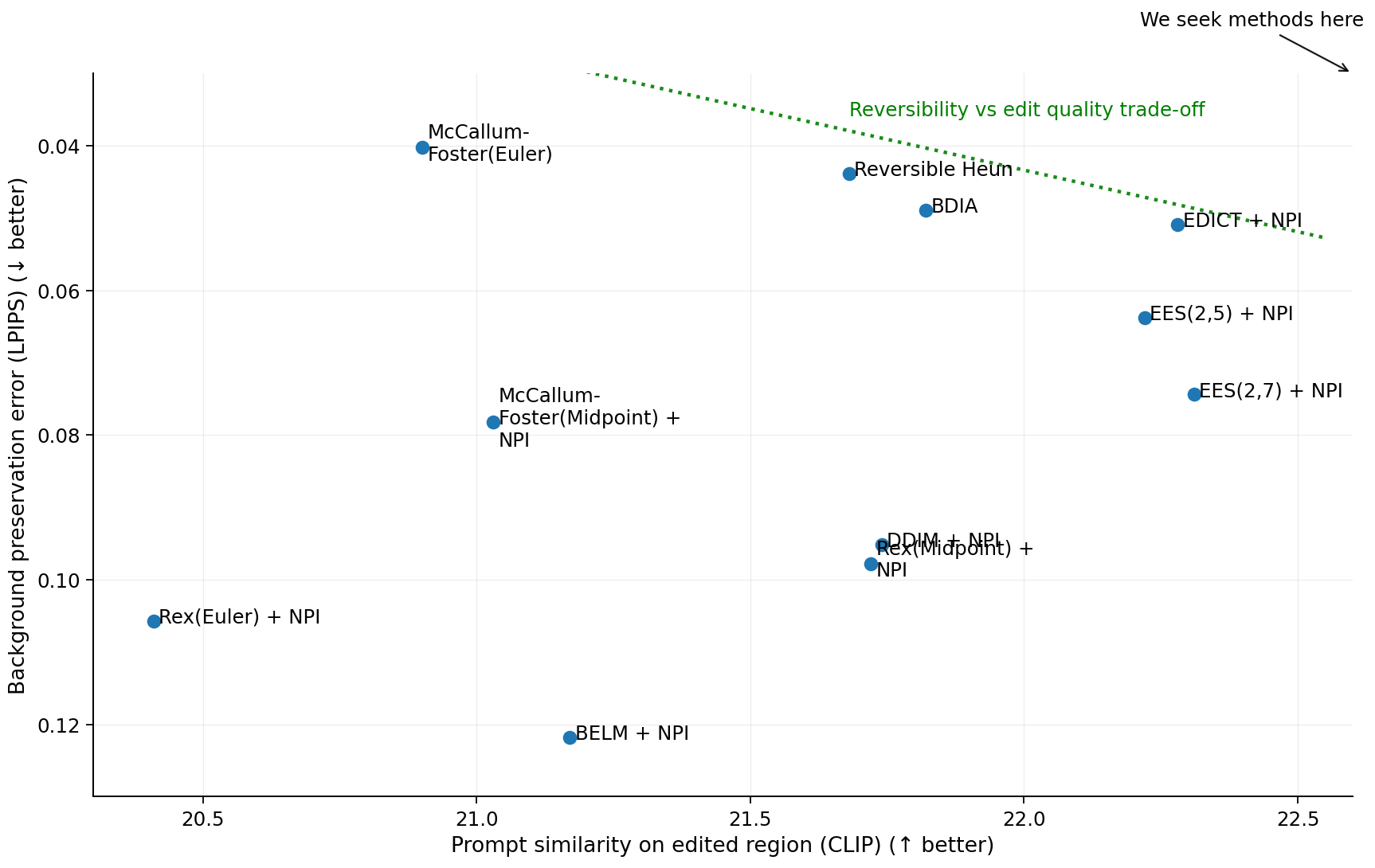}
    \caption{Equivalent of \cref{fig:metrics_trade_off} with Stable Diffusion 1.5 swapped for Stable Diffusion XL. Evaluated on Category 0 of PIE-Bench.}
    \label{fig:sdxl}
\end{figure}

\begin{table}[tbp]
\caption{
Equivalent of \cref{tab:solver_metrics} and \cref{tab:clip_whole_sd_vs_smooth} with Stable Diffusion 1.5 swapped for Stable Diffusion XL and NPI replacing Smooth Diffusion.
(\textbf{best}, \underline{second-best}, $^{\star}$ third-best among (near) reversible solvers.)
}
\label{tab:main_solver_metrics_sdxl_npi}
\vskip 0.10in
\centering
\scriptsize
\renewcommand{\arraystretch}{1.08}
\setlength{\tabcolsep}{2.6pt}
\resizebox{\linewidth}{!}{%
\begin{tabular}{@{}lccccc|ccc@{}}
\toprule
\textbf{(Near) Reversible Solver} &
\multicolumn{5}{c|}{\textbf{Small Edits}} &
\multicolumn{3}{c}{\textbf{Large Edits}} \\
\cmidrule(lr){2-6}
\cmidrule(lr){7-9}
&
\multicolumn{2}{c}{\textbf{Background Preservation}} &
\multicolumn{3}{c|}{\textbf{Edit Quality}} &
\multicolumn{3}{c}{\textbf{Edit Quality}} \\
\cmidrule(lr){2-3}
\cmidrule(lr){4-6}
\cmidrule(lr){7-9}
&
\textbf{PSNR}$\,\uparrow$ &
\textbf{LPIPS}$\times 10^3\,\downarrow$ &
\textbf{CLIP (Edited)}$\,\uparrow$ &
\textbf{PickScore}$\,\uparrow$ &
\textbf{ImageReward}$\,\uparrow$ &
\textbf{CLIP (Whole)}$\,\uparrow$ &
\textbf{PickScore}$\,\uparrow$ &
\textbf{ImageReward}$\,\uparrow$ \\
\midrule

EDICT$^{\dagger}$
& 30.17 & 50.89 & \underline{22.28} & 21.09$^{\star}$ & \underline{0.03}
& 23.59$^{\star}$ & 20.01$^{\star}$ & -0.49$^{\star}$ \\

Reversible Heun
& 31.66$^{\star}$ & \underline{43.88} & 21.68 & 20.97 & -0.06
& 21.06 & 19.00 & -1.35 \\

BDIA
& \underline{32.01} & 48.90$^{\star}$ & 21.82 & 21.00 & -0.15
& 19.93 & 18.73 & -1.58 \\

BELM$^{\dagger}$
& 27.03 & 121.78 & 21.17 & 20.22 & -0.34
& 18.37 & 18.09 & -1.67 \\

McCallum-Foster (Euler)
& \textbf{33.25} & \textbf{40.22} & 20.90 & 20.78 & -0.37
& 13.13 & 17.42 & -2.15 \\

McCallum-Foster (Midpoint)$^{\dagger}$
& 30.77 & 78.19 & 21.03 & 20.57 & -0.38
& 13.96 & 17.36 & -2.16 \\

Rex (Midpoint)$^{\dagger}$
& 27.26 & 97.75 & 21.72 & 20.52 & -0.16
& 20.30 & 18.36 & -1.57 \\

Rex (Euler)$^{\dagger}$
& 26.74 & 105.70 & 20.41 & 20.08 & -0.52
& 14.39 & 17.14 & -2.16 \\

\textbf{EES(2,5)}$^{\dagger}$
& 29.55 & 63.72 & 22.22$^{\star}$ & \underline{21.21} & 0.01$^{\star}$
& \underline{25.01} & \underline{21.07} & \underline{0.12} \\

\textbf{EES(2,7)}$^{\dagger}$
& 28.48 & 74.39 & \textbf{22.31} & \textbf{21.26} & \textbf{0.08}
& \textbf{25.11} & \textbf{21.09} & \textbf{0.18} \\

\bottomrule
\end{tabular}%
}

\vspace{0.03in}
\begin{minipage}{\linewidth}
\scriptsize
Solvers with $^{\dagger}$ are reported using NPI for \textit{small edits}. For \textit{large edits}, we report all solvers with NPI.
\end{minipage}

\vskip -0.10in
\end{table}

\subsection{Varying the Number of Function Evaluations}
\label{sec:nfe}
In \cref{fig:nfe}, we report how the trade-off described in \cref{fig:metrics_trade_off} behaves under varied model evaluations budgets.

\begin{figure}[tbp]
    \centering
    \includegraphics[width=\linewidth]{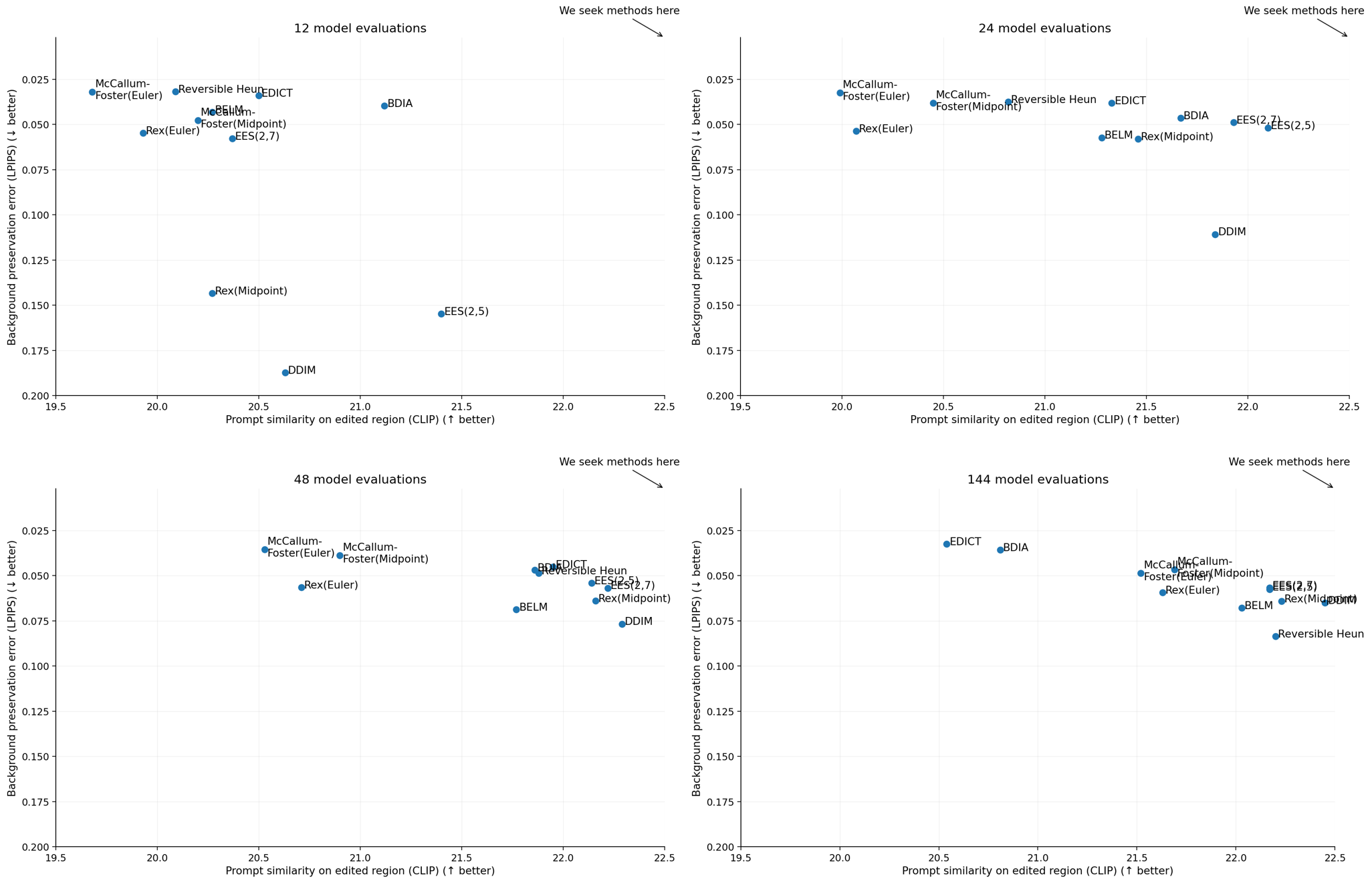}
    \caption{Equivalent of \cref{fig:metrics_trade_off} for 12, 24, 48 (original), and 144 model evaluations. Evaluated on Category 0 of PIE-Bench. As in \cref{sec:experiments_editing}, we use Smooth Diffusion on all solvers except Reversible Heun, BDIA, and McCallum-Foster(Euler).}
    \label{fig:nfe}
\end{figure}

\subsection{EDICT and BDIA hyperparameters}
\label{sec:edict_bdia_hyperparameters}

EDICT and BDIA include a hyperparameter controlling how strongly the solver follows the edited prompt, with recommended values in \([0.9,1.0]\). One might expect that tuning these values could recover stability under large edits. However, we find that varying \(\gamma\) (BDIA) and \(p\) (EDICT) does not resolve the failures above. \cref{fig:EDICT_BDIA_hyperparams_large_edits} illustrates this on an example from the task in \cref{sec:large_edits_and_stability}.

\begin{figure}[tbp]
    \centering
    \includegraphics[width=\linewidth]{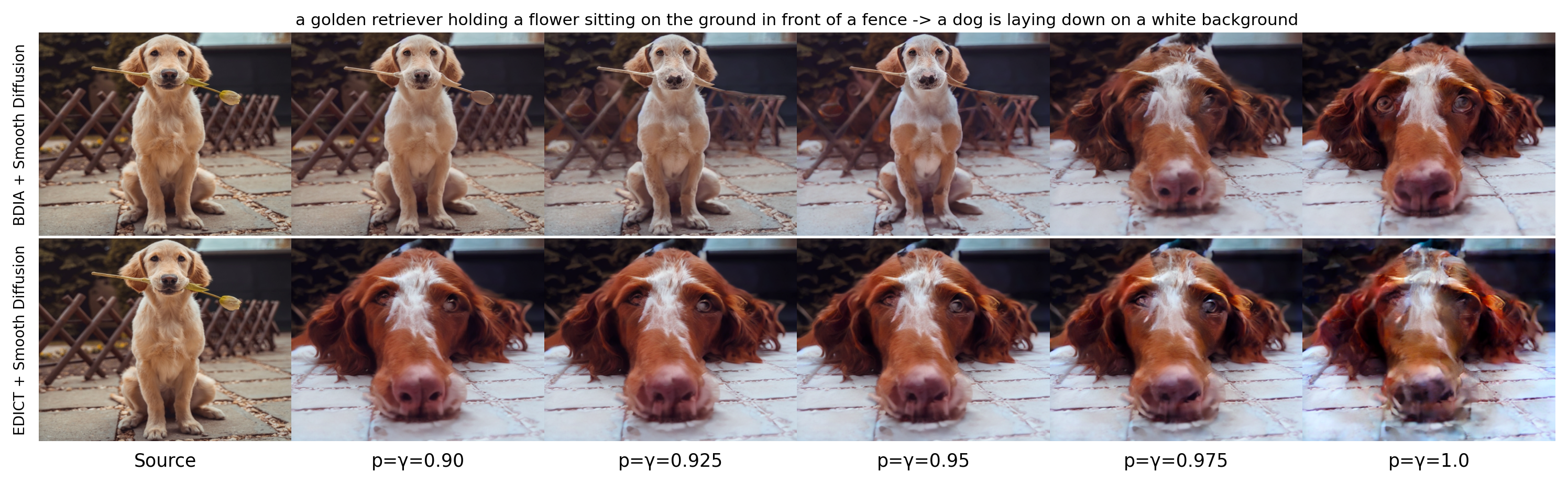}
    \caption{The impact of the BDIA hyperparameter \(\gamma\) and the EDICT hyperparameter \(p\) on the large-prompt-deviations task in \cref{sec:large_edits_and_stability}. We use Smooth Diffusion for all samples.}
    \label{fig:EDICT_BDIA_hyperparams_large_edits}
\end{figure}

\subsection{Terminal Latent Distribution under Inversion}
\label{sec:latent_terminal}

Motivated by recent observations that diffusion inversion can suffer from numerical singularities, causing the inverted terminal latent \(x_T\) to deviate substantially from the intended isotropic Gaussian prior \(\mathcal{N}(0,I)\) \cite{blasingame2025rex, staniszewski2024there}, we visualise the empirical distribution of the terminal latents \(x_T\) for the image shown in \cref{fig:grid_retriever}. \cref{fig:histograms} compares all solvers across three NFE settings against the standard normal density. While most methods including EES remain well aligned with \(\mathcal{N}(0,1)\), EDICT and BDIA become increasingly unstable as the NFE increases, with the variance of the inverted latent growing by several orders of magnitude. This provides a possible explanation for their degradation in CLIP score at higher NFE, as shown in \cref{sec:nfe}. See \citet[Appendix~J]{blasingame2025rex} for further discussion about this phenomenon.

\begin{figure}[tbp]
    \centering
    \includegraphics[
        width=\linewidth,
        height=0.82\textheight,
        keepaspectratio
    ]{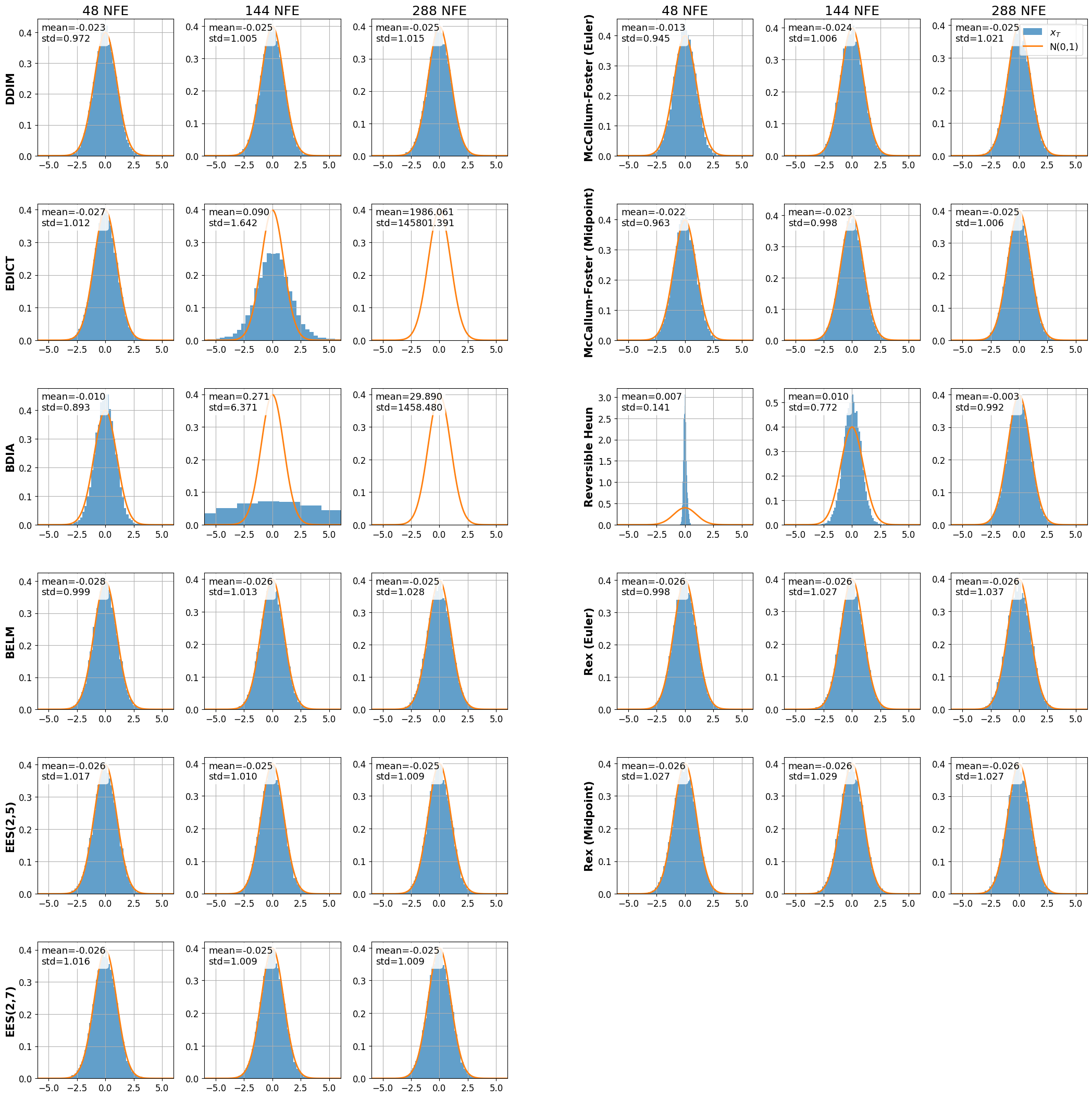}
    \caption{
    Empirical distribution of the inverted terminal latents \(x_T\) for the image in \cref{fig:grid_retriever}. Columns correspond to different NFE settings, and rows correspond to different solvers. Most solvers remain close to the Gaussian prior, whereas EDICT and BDIA exhibit extreme variance at larger NFE.
    }
    \label{fig:histograms}
\end{figure}

\subsection{Varying the Guidance Scale}
\label{sec:guidance_scale}

In \cref{fig:guidance_scale}, we study how the edit quality and background preservation trade-off changes as the guidance scale is varied. We evaluate guidance scales \(1.0\), \(2.0\), and \(3.0\) on Category~0 of PIE-Bench, using Smooth Diffusion for all solvers. As expected, increasing the guidance scale generally improves prompt alignment, but worsens background preservation.

\begin{figure}[tbp]
    \centering
    \includegraphics[width=\linewidth]{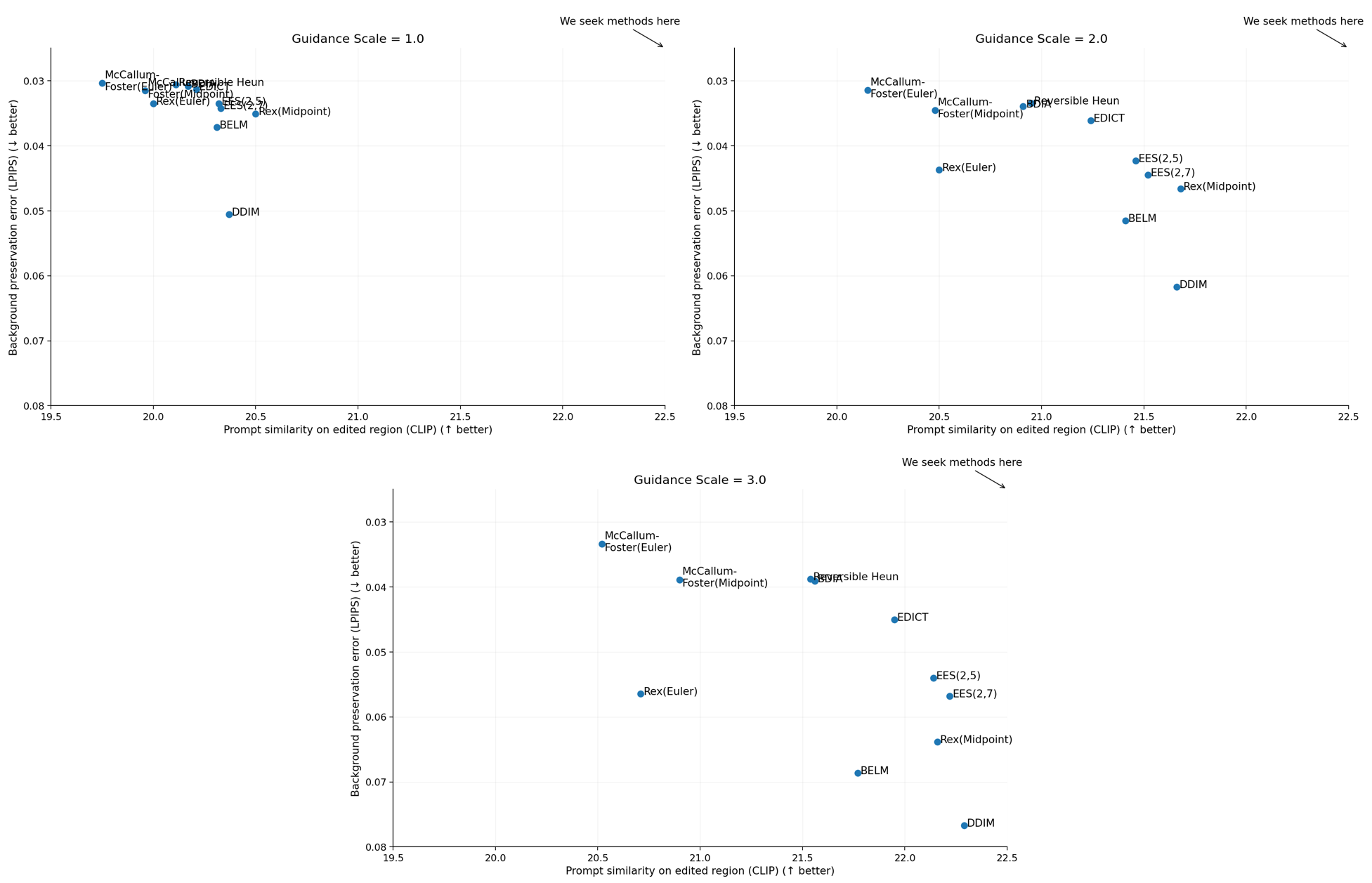}
    \caption{
    Equivalent of \cref{fig:metrics_trade_off} with the guidance scale varied from \(1.0\) to \(3.0\). Evaluated on Category~0 of PIE-Bench, using Smooth Diffusion for all solvers. Increasing the guidance scale improves prompt alignment but generally worsens background preservation.
    }
    \label{fig:guidance_scale}
\end{figure}

\section{Experiments Supplement}
\label{appendix:experiments}
\subsection{Pretrained Models}
\label{appendix:pretrained_models}
We use the following pretrained models in all experiments:

\textbf{Stable Diffusion 1.5}: 

\url{https://huggingface.co/stable-diffusion-v1-5/stable-diffusion-v1-5}

\textbf{Smooth Diffusion}: 

\url{https://huggingface.co/shi-labs/smooth-diffusion-lora}

\textbf{Stable Diffusion XL}: 

\url{https://huggingface.co/stabilityai/stable-diffusion-xl-base-1.0}

\subsection{Code Implementations}
\label{appendix:code_implementations}
We implement BELM, EDICT, DDIM, and BDIA using the public BELM codebase: \url{https://github.com/zituitui/BELM/tree/main}. We implement Rex using the code provided in \citet[Appendix H]{blasingame2025rex}. We implement McCallum–Foster directly from the update equations in \cref{appendix:mccallum_foster} using Euler and Midpoint base solvers.

We note that the BELM repository contains a bug in the BDIA implementation, which has been reported in an open issue (\#13). We correct this error in our code.

\subsection{Wall-Clock Time Comparison}

As stated in the image editing setup in \cref{sec:experiments}, we align all methods by the number of function evaluations (NFE), so that higher-order solvers do not receive an unfair compute advantage. Wall-clock time is dominated by model evaluations, while the numerical overhead of each solver is comparatively small. Using the Smooth Diffusion SD 1.5 checkpoint introduces negligible overhead. We report wall-clock timings for the standard 48-evaluation editing setup in \cref{tab:wall_clock}, measured on a single NVIDIA GeForce RTX 4070 Ti GPU.

\begin{table}[tbp]
\centering
\caption{
Wall-clock time for the standard 48-NFE editing setup. All methods use the same total number of model evaluations. Timings are reported for SD v1.5 with and without Smooth Diffusion.
}
\label{tab:wall_clock}
\resizebox{\linewidth}{!}{
\begin{tabular}{lcccc}
\toprule
Solver & Steps & Model evals/step & SD v1.5 time (s) & SD v1.5 + Smooth Diffusion time (s) \\
\midrule
DDIM & 48 & 1 & 5.35 & 5.62 \\
EDICT & 24 & 2 & 5.26 & 5.53 \\
BDIA & 48 & 1 & 5.28 & 5.54 \\
Rev. Heun & 24 & 2 & 5.39 & 5.67 \\
\midrule
EES(2,5) & 16 & 3 & 5.35 & 5.61 \\
EES(2,7) & 12 & 4 & 5.33 & 5.60 \\
\bottomrule
\end{tabular}
}
\end{table}

\subsection{Supplemental Figures}
We present the following supplemental figures.

\begin{itemize}
    \item \cref{fig:edit_grid_retriever_extended}: Extended small-edit example (retriever).
    \item \cref{fig:edit_grid_storm_trooper_extended}: Extended small-edit example (storm trooper).
    \item \cref{fig:edit_grid_owl_extended}: Extended small-edit example (owl).
    \item \cref{fig:large_edits_retriever_extended}: Extended large-edit example (retriever).
\end{itemize}

\begin{figure}[tbp]
    \centering
    \includegraphics[width=\linewidth]{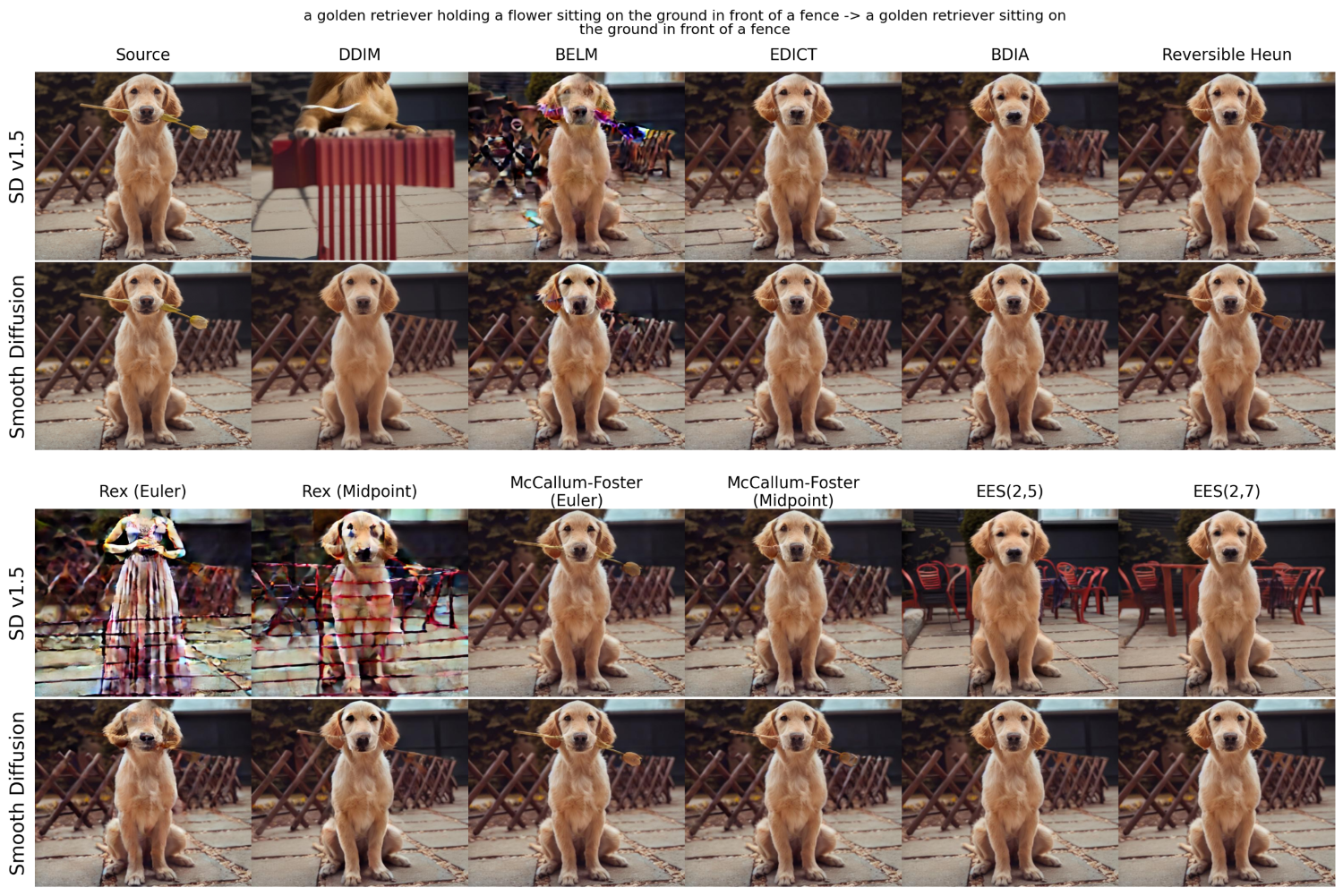}
    \caption{Extended grid for examples in \cref{fig:example_figure}. Qualitative editing comparison for all solvers under Stable Diffusion 1.5 (top row) and Smooth Diffusion (bottom row). Smooth Diffusion improves background preservation (e.g., the fence) while maintaining the intended edit.}
    \label{fig:edit_grid_retriever_extended}
\end{figure}
\begin{figure}[tbp]
    \centering
    \includegraphics[width=\linewidth]{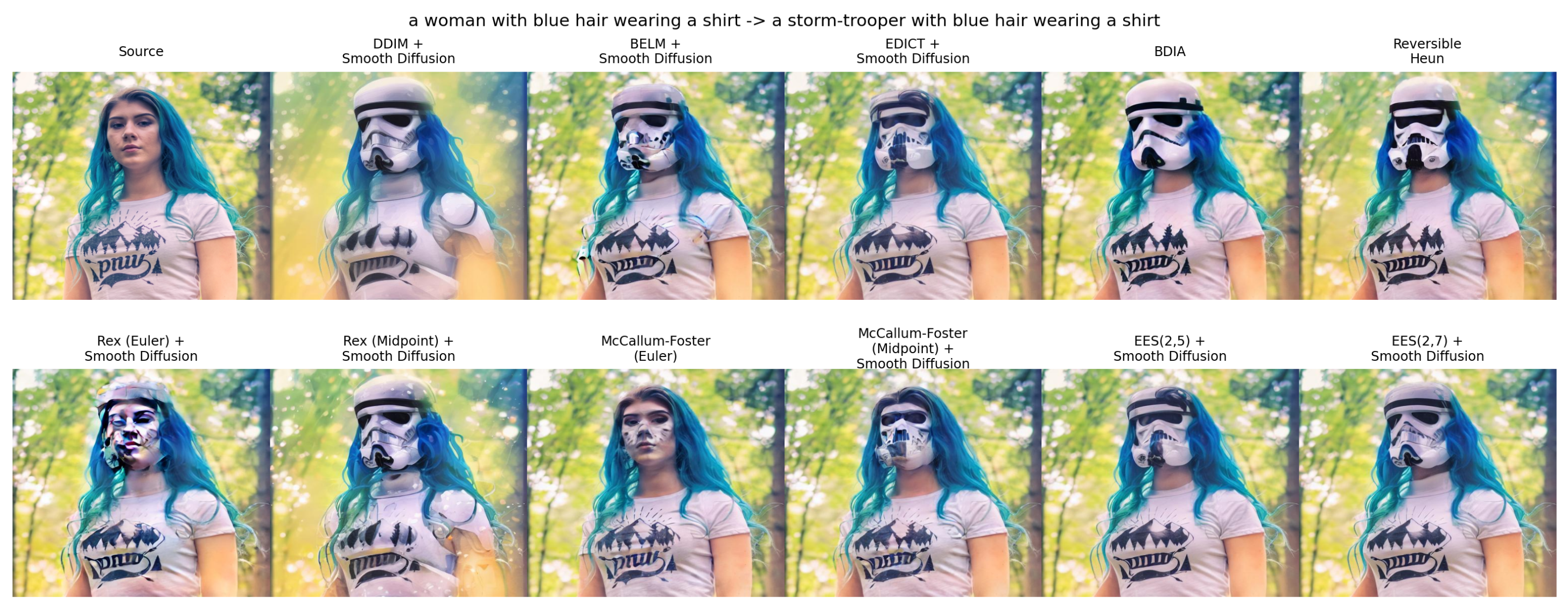}
    \caption{Extended grid for \cref{fig:edit_grid_storm_trooper}. We include all solvers for comparison.}
    \label{fig:edit_grid_storm_trooper_extended}
\end{figure}
\begin{figure}[tbp]
    \centering
    \includegraphics[width=\linewidth]{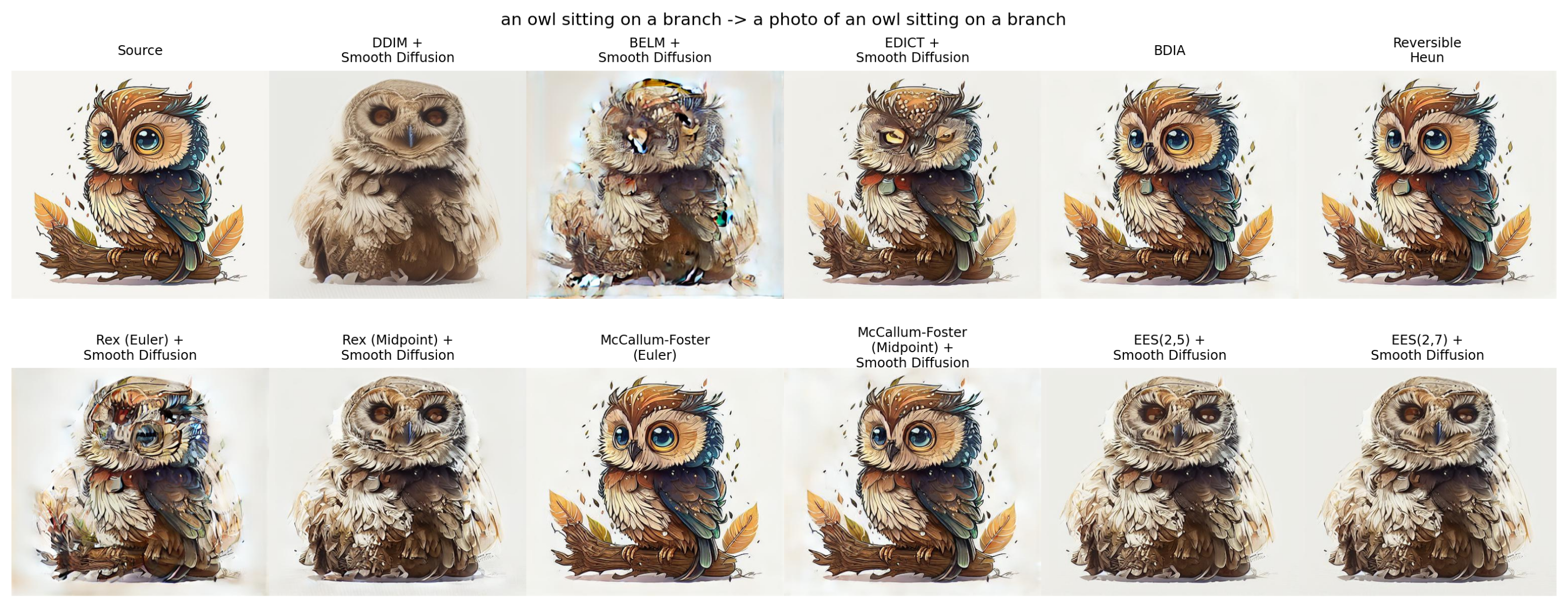}
    \caption{Extended grid for examples in \cref{fig:example_figure}. We include all solvers for comparison.}
    \label{fig:edit_grid_owl_extended}
\end{figure}
\begin{figure}[tbp]
    \centering
    \includegraphics[width=\linewidth]{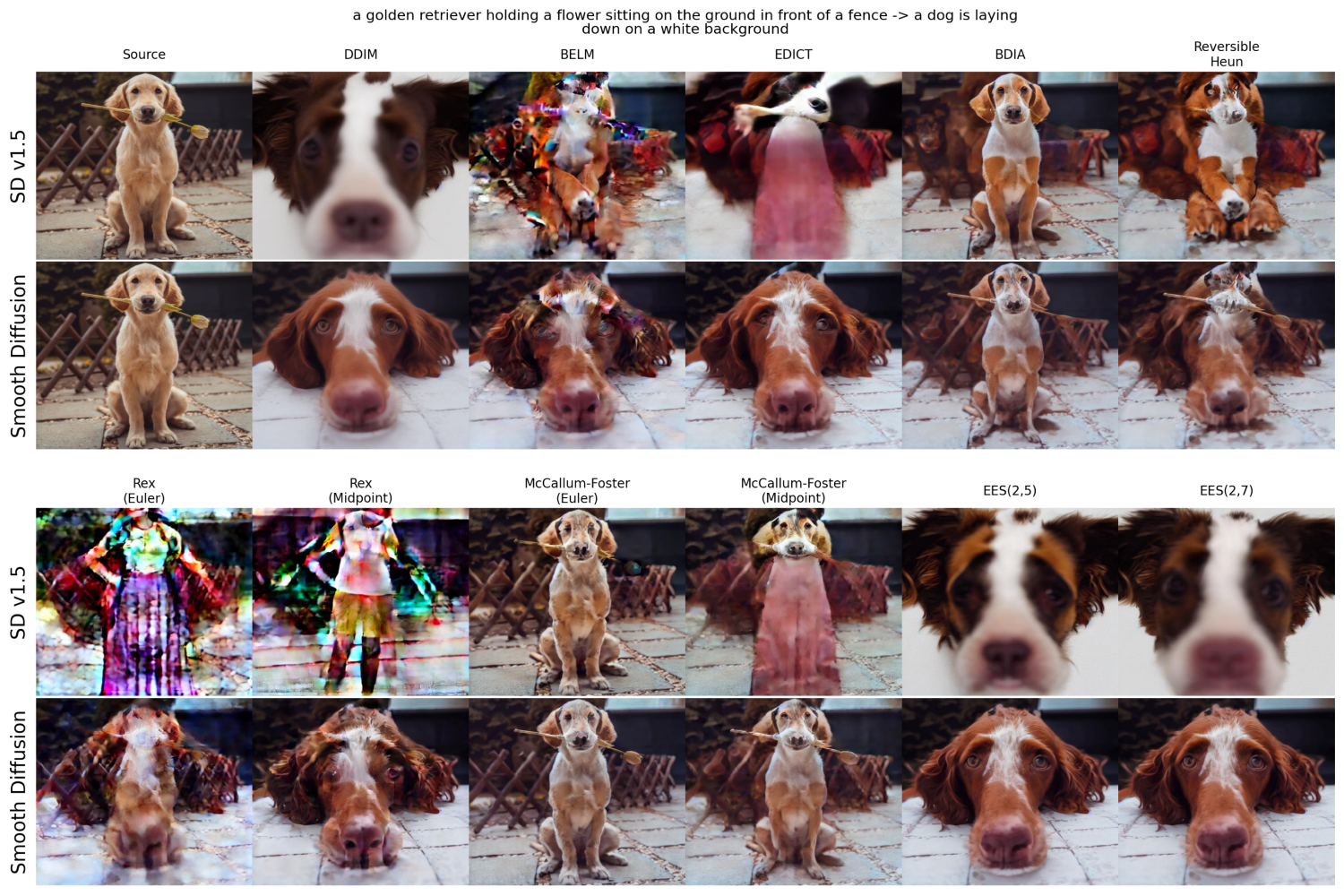}
    \caption{Extended grid for examples in  \cref{fig:example_figure}. Non-EES reversible solvers often fail when the edit requires a large trajectory deviation. We show results under Stable Diffusion v1.5 (top rows) and Smooth Diffusion (bottom rows).}
    \label{fig:large_edits_retriever_extended}
\end{figure}

\subsection{Dataset}
We use PIE-Bench \citep{ju2023direct} for image editing evaluation. Its annotations are licensed under CC-BY-NC-SA 4.0. The qualitative examples in this paper are shown for non-commercial academic use only.

\end{document}